\documentclass[twoside]{article}
\usepackage[hmarginratio=1:1,top=32mm,columnsep=20pt]{geometry}
\usepackage[font=it]{caption}
\usepackage{paralist}
\usepackage{multicol}

\usepackage{lettrine}

\usepackage{abstract}

\usepackage{titlesec}
\renewcommand\thesection{\Roman{section}}
\titleformat{\section}[block]{\large\scshape\centering}{\thesection.}{1em}{}

\usepackage{fancyhdr}
\usepackage{amsthm}
\usepackage{amsmath}
\usepackage{amssymb}

\makeatletter
\let\start@align@nopar\start@align
\let\start@gather@nopar\start@gather
\let\start@multline@nopar\start@multline
\long\def\start@align{\par\start@align@nopar}
\long\def\start@gather{\par\start@gather@nopar}
\long\def\start@multline{\par\start@multline@nopar}
\makeatother

\usepackage{epsfig}
\usepackage{epstopdf}

\usepackage{graphicx}
\usepackage{subfigure}

\usepackage{algorithm}
\usepackage{algorithmic}

\usepackage{multirow}
\usepackage{makecell}

\usepackage{xcolor}

\usepackage{comment}

\usepackage[comma, sort]{natbib}

\usepackage{setspace}

\newcommand{\SimpleTitle}[3]
{
	\title{ \vspace{-12mm}%
		\fontsize{24pt}{10pt}\selectfont
		\textbf{#1}
	}	
	\author{%
		\large
		#2 \\[2mm]
		\normalsize	Computer Science Institute, Zhejiang University\\
		\normalsize	{\{#3\}}
		\vspace{-0mm}
	}
	\date{\today}
}
\usepackage{rotating} 
\usepackage{amssymb}
\DeclareMathAlphabet\mathbfcal{OMS}{cmsy}{b}{n}
\makeatletter
\newcommand{\AB}{\mathbf{A}}
\newcommand{\BB}{\mathbf{B}}
\newcommand{\CB}{\mathbf{C}}
\newcommand{\DB}{\mathbf{D}}
\newcommand{\EB}{\mathbf{E}}
\newcommand{\FB}{\mathbf{F}}
\newcommand{\GB}{\mathbf{G}}

\newcommand{\IB}{\mathbf{I}}

\newcommand{\KB}{\mathbf{K}}
\newcommand{\LB}{\mathbf{L}}

\newcommand{\NB}{\mathbf{N}}
\newcommand{\OB}{\mathbf{O}}
\newcommand{\PB}{\mathbf{P}}
\newcommand{\QB}{\mathbf{Q}}

\newcommand{\SB}{\mathbf{S}}

\newcommand{\UB}{\mathbf{U}}
\newcommand{\VB}{\mathbf{V}}
\newcommand{\WB}{\mathbf{W}}
\newcommand{\XB}{\mathbf{X}}
\newcommand{\YB}{\mathbf{Y}}
\newcommand{\ZB}{\mathbf{Z}}
\newcommand{\OmegaB}{\mathbf{\Omega}}
\newcommand{\zeroB}{\mathbf{0}}

\newcommand{\RBB}{\mathbb{R}}

\newcommand{\FM}{\mathcal{F}}
\newcommand{\MM}{\mathcal{M}}
\newcommand{\HM}{\mathcal{H}}
\newcommand{\OM}{\mathcal{O}}
\newcommand{\PM}{\mathcal{P}}
\newcommand{\TM}{\mathcal{T}}
\newcommand{\VM}{\mathcal{V}}

\newcommand{\AT}{{\mathbfcal{A}}}
\newcommand{\BT}{{\mathbfcal{B}}}
\newcommand{\ET}{{\mathbfcal{E}}}
\newcommand{\GT}{{\mathbfcal{G}}}
\newcommand{\ST}{{\mathbfcal{S}}}

\newcommand{\GTi}{{\GT_{(i)}}}
\newcommand{\HT}{{\mathbfcal{H}}}
\newcommand{\RT}{{\mathbfcal{R}}}
\newcommand{\PT}{\mathbf{\PM}}
\newcommand{\XT}{\mathbfcal{X}}
\newcommand{\YT}{\mathbfcal{Y}}
\newcommand{\XTk}{{\mathbfcal{X}^{(k)}}}
\newcommand{\PO}{{\PT_\Omega}}
\newcommand{\rank}{\mathop{\rm rank}}
\newcommand{\mrank}{\mathop{\rm rank^{vec}}}
\newcommand{\range}{\mathop{\rm range}}
\newcommand{\OXT}{{\overline{\mathbfcal{X}}}}
\newcommand{\OYT}{{\overline{\mathbfcal{Y}}}}
\newcommand{\OZT}{{\overline{\mathbfcal{Z}}}}
\newcommand{\OXTk}{{\overline{\XTk}}}

\newcommand{\MMr}{{\mathcal{M}_r}}
\newcommand{\OMM}{{\overline{\MMr}}}

\newcommand{\TOMM}{{T_{\OXT}\OMM}}
\newcommand{\HX}{{\HM_{\OXT}}}
\newcommand{\VX}{{\VM_{\OXT}}}
\newcommand{\st}{\mathop{\rm St}}
\newcommand{\barf}{\bar{f}}

\newcommand{\xiOX}{ {\xi_{\OXT}} }
\newcommand{\xiG}{{\xi_\GT}}
\newcommand{\xiUi}{{\xi_{i}}}

\newcommand{\etaOX}{ {\eta_{\OXT}} }

\newcommand{\etaG}{{\eta_\GT}}
\newcommand{\etaUi}{{\eta_{i}}}
\newcommand{\etaU}{{\eta_{1}}}
\newcommand{\etaV}{{\eta_{2}}}
\newcommand{\etaW}{{\eta_{3}}}

\newcommand{\sbt}{{\text{s.b.t}}}
\newcommand{\spann}{\mathop{\rm span}}
\newcommand{\dist}{\mathop{\rm dist}}
\newcommand{\tr}{\mathop{\rm trace}}
\newcommand{\symm}{\mathop{\rm sym}}
\newcommand{\skewm}{\mathop{\rm skw}}
\newcommand{\uf}{\mathop{\rm uf}}
\newcommand{\grad}{\mathop{\rm grad}}

\newcommand{\clr}[1]{{\color[rgb]{0,0,0.6} #1}}
\newtheorem{theorem}{Theorem}
\newtheorem{lemma}[theorem]{Lemma}
\newtheorem{proposition}[theorem]{Proposition}

\theoremstyle{remark}
\newtheorem{remark}{Remark}
\SimpleTitle{Riemannian Tensor Completion \\[1mm] with Side Information}% title
{
\textsc{Tengfei Zhou} \, 
\textsc{Hui Qian}\thanks{Corresponding author.} \, 
\textsc{Zebang Shen} \, 
\textsc{Congfu Xu}} % authors
{zhoutengfei\_zju, qianhui, shenzebang, xucongfu@zju.edu.cn} % email 

\begin{document}
\maketitle
\thispagestyle{fancy}

\begin{abstract}
	By restricting the iterate on a nonlinear manifold, the recently proposed Riemannian optimization methods prove to be both efficient and effective in low rank tensor completion problems. 
	However, existing methods fail to exploit the easily accessible side information, due to their format mismatch.
	Consequently, there is still room for improvement in such methods.
	To fill the gap, in this paper, a novel Riemannian model is proposed to organically integrate the original model and the side information
	by overcoming their inconsistency.
	For this particular model, an efficient Riemannian conjugate gradient descent solver is devised based on a new metric that captures the curvature of the objective.
	Numerical experiments suggest that our solver is more accurate than the state-of-the-art without compromising the efficiency.
\end{abstract}

\section{Introduction}
Low Rank Tensor Completion (LRTC) problem, which aims to recover a \emph{tensor} from its linear measurements, arises naturally in many artificial intelligence applications. In hyperspectral image inpainting, LRTC is applied to interpolate the unknown pixels based on the partial observation ~\clr{\cite{xu2015parallel}}. In recommendation tasks, LRTC helps users find interesting items under specific contexts such as locations or time~\clr{\cite{liu2015cot}}. 
In computational phenotyping, one adopts LRTC to discovery phenotypes in heterogeneous electronic health records~\clr{\cite{wang2015rubik}}.

\noindent \textbf{Euclidean Models:}~LRTC can be formulated by a variety of optimization models over the \emph{Euclidean} space. Amongst them, convex models that encapsulate LRTC as a regression problem penalized by a tensor nuclear norm are the most popular and well-understood~\clr{\cite{romera2013new,zhang2014hybrid}}. Though most of them have sound theoretical guarantees~\clr{\cite{zhang2016exact,chen2013exact,yuan2015tensor}}, in general, their solvers are ill-suited for large tensors because these procedures usually involve Singular Value Decomposition (SVD) of huge matrices per iteration~\clr{\cite{liu2013tensor}}.
Another class of Euclidean models is formulated as the decomposition problem that factorizes a low rank tensor into small factors~\clr{\cite{jain2014provable,filipovic2015tucker,xu2015parallel}}. Many solvers for such decomposition based model have been proposed to recover large tensors, and low per-iteration computational cost is illustrated~\clr{\cite{beutel2014flexifact,liu2014factor,smith2016exploration}}.

\noindent \textbf{Riemannian Models:}~LRTC can also be modeled by nonconvex optimization constrained on \emph{Riemannian} manifolds ~\clr{\cite{kressner2014low,kasai2016low}}, which is easily handled by many manifold based solvers~\clr{\cite{absil2009optimization}}. Empirical comparison has shown that Riemannian solvers use significantly less CPU time to recover the underlying tensor in contrast to the Euclidean solvers~\clr{\cite{kasai2016low}}. The main reason resides in that such solvers avoid SVD of huge matrices by explicitly exploiting the geometrical structure of LRTC, which makes them more suitable for massive problem.

Of all the Riemanian models, two search spaces, fix multi-linear rank manifold~\clr{\cite{kressner2014low}} and Tucker manifold~\clr{\cite{kasai2016low}}, are usually employed. The former is a sub-manifold of Euclidean space, and the latter is a quotient manifold induced by the \emph{Tucker decomposition}. Generally, quotient manifold based solvers have higher convergence rates because it is usually easier to design a pre-conditioner for them~\clr{\cite{kasai2016low,mishra2016riemannian}}.

\noindent \textbf{Side Information:}~In the Euclidean models of LRTC, \emph{side information} is helpful in improving the accuracy ~\clr{\cite{narita2011tensor,AcKoDu11,beutel2014flexifact}}. 
One common form of the side information is the \emph{feature matrix}, which measures the statistical properties of tensor modes~\clr{\cite{kolda2009tensor}}. For example, in Netflix tasks, feature matrix can be built from the demography of users~\clr{\cite{bell2007lessons}}. Another form is the \emph{similarity matrix}, which quantifies the resemblance between two entities of a tensor mode. For instance, the social network generates the similarity matrix by utilizing the correspondence between users~\clr{\cite{rai2015leveraging}}.
%which is available in many datasets can be naturally treated as the source of similarity matrix since two users with social correspondence may have similar features~\clr{\cite{rai2015leveraging}}. 
In practice, these two matrices can be transformed to each other, and we only consider the feature matrix case throughout this paper. 

However, as far as we know, side information has not been incorporated in any Riemannian model. The first difficulty lies in the model design.
Fusing the side information into the Riemannian model inevitably compromises the integrity of the low rank tensor due to the compactness of the manifold.
% {Though the feature matrices describe latent factors of a tensor, it is hard to formulate their relation in a model}.
%
The second difficulty results from the solver design. Incorporating the side information may aggravate the ill-conditioning of LRTC problem and degenerates the convergence significantly. 
%To overcome this difficulty, in Riemannian optimization framework,  implies that one has to devise a highly applicable Riemannian metric by exploiting the problem structure.

\noindent \textbf{Contributions:}~To address these difficulties, a novel Riemannian LRTC method is proposed from the perspective of both model and solver designs. By exploring the relation between the subspace spanned by the tensor fibers and the column space of the feature matrix, we explicitly integrate the side information in a compact way. Meanwhile, a first order solver is devised under the manifold optimization framework. To ease the ill-conditioning, we design a novel metric based on an approximated Hessian of the cost function. The metric implicitly induce an adaptive preconditioner for our solver. Empirical studies illustrate that  our method achieves much more accurate solutions within comparable processing time than the state-of-the-art.

\section{Notations and Preliminaries}
In this paper, we only focus on the $3$rd order tensor, but generalizing our method to high order is straight forward.  We use the notation $\XB \in \RBB^{n \times m}$ to denote a matrix, and  the notation $\XT \in \RBB^{n_1 \times \cdots \times n_d}$ to denote a $d$-th order tensor. We also denote by $\XT(i_1,\cdots,i_d)$ the element in position $(i_1,\cdots,i_d)$ of $\XT$. 
For many cases, we use curly braces with indexes to simplify the notation. For example, $\{\OB_i\}_{i=1}^3$ is used to denote $\OB_1,\OB_2,\OB_3$, and $\{\UB_i \OB_i\}_{i=1}^3$ refer to $\UB_1\OB_1,\UB_2\OB_2,\UB_3\OB_3$. 

\noindent \textbf{Mode-$k$ Fiber and matricization:} A fiber of a tensor is obtained by varying one index while fixing the others, i.e. $\XT(i_1,\cdots,i_{k-1},:,i_{k+1},\cdots,i_d)$ is the mode-$k$ fiber of a $d$-th order tensor $\XT$.
Here we use the colon to denote $\{1,\ldots,n_k\}$.
A \emph{mode-$k$ matricization} $\XT_{(k)} \in \RBB^{ n_k \times (n_1\cdots n_{k-1} n_{k+1}\cdots n_d) }$ of a tensor $\XT$ is obtained by arranging the mode-$n$ fibers of $\XT$ so that each of them is a column of $\XT_{(k)}$~\clr{\cite{kolda2009tensor}}
%The \emph{mode-$k$ matricization} of tensor $\XT$ is denoted by $\XT_{(k)}$. Note that $\XT_{(k)} \in \RBB^{ n_k \times (n_1\cdots n_{k-1} n_{k+1}\cdots n_d) }$ is a matrix obtained by arranging the mode-$n$ fibers of $\XT$ so that each of them is a column of $\XT_{(n)}$~\clr{\cite{kolda2009tensor}}.
The mode-$k$ product of tensor $\XT$ and matrix $\AB$ is denoted by $\XT \times_k \AB$, whose mode-$k$ matricization can be expressed as $(\XT \times_k \AB)_{(k)} = \AB\XT_{(k)}$. For $3$rd order tensor $\XT$ and matrix $\AB_1, \AB_2, \AB_3$, we use $\XT \times_{i=1}^3 \AB_i$ to denote $\XT \times_1 \AB_1 \times_2 \AB_2 \times_3 \AB_3$.

%\noindent \textbf{Inner product and norm:}~The inner product of two tensor of the same size is defined by $\langle \XT, \YT \rangle = \sum_{i_1,\cdots,i_d} \XT(i_1,\cdots,i_d)\YT(i_1,\cdots,i_d)$. For the matrix case, $\langle \AB,\BB \rangle = \tr(\AB^\top\BB)$. The Frobenius norm of both matrix and tensor are denoted by $\|\cdot\|_F$ where $\|\AB\|_F = \sqrt{\langle \AB, \AB \rangle}$ and $\|\XT\|_F = \sqrt{\langle \XT, \XT \rangle}$.

\noindent \textbf{Inner product and norm:}~The inner product of two tensors with the same size is defined by $\langle \XT, \YT \rangle = \sum_{i_1,\cdots,i_d} \XT(i_1,\cdots,i_d)\YT(i_1,\cdots,i_d)$.
The Frobenius norm of a tensor $\XT$ is defined by $\|\XT\|_F = \sqrt{\langle \XT, \XT \rangle}$.

\noindent \textbf{Multi-linear rank and Tucker decomposition:}~
The multi-linear rank $\mrank(\XT)$ of a tensor $\XT \in \RBB^{n_1 \times n_2 \times n_3}$ is defined as a vector $( \rank({\XT_{(1)}}),\rank({\XT_{(2)}}),\rank({\XT_{(3)}}))$. 
If $\mrank(\XT) = (r_1, r_2, r_3) $, tucker decomposition factorizes $\XT$ into a small core tensor $\GT \in \RBB^{r_1 \times r_2 \times r_3}$ and three matrices $\UB_i\in\RBB^{n_i\times r_i}$ with orthogonal columns, that is $\XT = \GT \times_{i=1}^3 \UB_i$.
Note that, the tucker decomposition of a tensor is not unique. In fact, if $\XT = \GT \times_{i=1}^3 \UB_i$, we can easily obtain $\XT = \HT \times_{i=1}^3 \VB_i$, with $\HT = \GT \times_{i=1}^3 \OB^{\top}_i$, $\VB_i = \UB_i \OB_i$, where $\OB_i \in \RBB^{r_i\times r_i}$ is any orthogonal matrix. Thus, we obtain the equivalent class
{\small
	\[[\GT,\{\UB_i\}_{i=1}^3]
	\triangleq\{(\GT\times_{i=1}^3\OB_i^\top, \{\UB_i\OB_i\}_{i=1}^3 )| \OB_i^\top\OB_i = \IB_i\}.
	\]}
For simplicity, we denote $[\GT,\{\UB_i\}_{i=1}^3]$ by $[\XT]$, when $\XT = \GT \times_{i=1}^3 \UB_i$. Usually, $[\XT]$ is called the Tucker representation of $\XT$, while $\XT$ is call the tensor representation
of $[\XT]$. 
We also use $\overline{\XT}$ to denote a specific decomposition of $\XT$, additionally $\overline{\XT} \in [\XT]$.

\subsection{Search Space of Riemannian Models}
% a higher-order singular value decomposition.
The Tucker manifold that we used in our Riemannian model is a quotient manifold  induced by the Tucker decomposition. In order to lay the ground for Tucker manifold, we first describe its counterpart, the fix multi-rank manifold, which will be helpful in understanding the whole derivation.

% FMLR manifold
A fixed multi-linear rank manifold $\FM_r$ consists of tensors with the same fixed multi-linear rank. Specifically 
\[
\FM_r = \{\XT \in \RBB^{n_1 \times n_2 \times n_3}|\mrank(\XT)=r\}.
\]
%where $r=(r_1, r_2, r_3)$.

% Tucker manifold
To define the Tucker manifold, we first define a total space
\begin{equation}\label{eq:totalspace}
\MMr = \RBB^{r_1\times r_2 \times r_3} \times \st(r_1, n_1) \times \st(r_2, n_2) \times \st(r_3, n_3),
\end{equation}
in which $\st(r_i,n_i)$ is the \emph{Stiefel manifold} of a $n_i \times r_i$ matrix with orthogonal columns. Then, we can depict the Tucker manifold of multi-linear rank $r$ as follows.
\begin{equation}
\MMr /\sim \triangleq \left\{[\GT,\{\UB_i\}_{i=1}^3] | (\GT,\{\UB_i\}_{i=1}^3) \in \MMr\right\}.
\end{equation}
The Tucker manifold is a quotient manifold of the total space (\ref{eq:totalspace}). 
We use the abstract quotient manifold, rather than the concrete total space, as search space because the non-uniqueness of the Tucker decomposition is undesirable for optimization.
Note that such non-uniqueness will introduce more local optima into the minimization.
The relation of manifold $\FM_r$ and $\MMr/\sim$ is characterized as follows.
\begin{proposition} \small
	The quotient manifold $\MMr / \sim$ is diffeomorphic to the fix multi-linear rank  manifold $\FM_r$, with
	diffeomorphism $\rho(\cdot)$ from $\FM_r$ to $\MMr / \sim$ defined by $\rho(\XT) = [\GT,\{\UB_i\}_{i=1}^3]$
	where $[\GT,\{\UB_i\}_{i=1}^3]$ is the tucker representation of $\XT$.
	\label{prop:diffeom}
\end{proposition}
This proposition says that each tensor $\XT \in \FM_r$ can be represented by a unique equivalent class $[\GT,\{\UB_i\}_{i=1}^3] \in \MM/\sim$ and vice-versa. 

\subsection{Vanilla Riemannian Tensor Completion}
The purest incarnation of Riemannian tensor completion model is the Riemannian model over the fix multi-linear rank manifold.
Let $\RT\in\RBB^{n_1\times n_2 \times n_3}$ be a partially observed tensor. 
Let $\Omega$ be the set which contains the indices of observed entries.
The model can be expressed as:
\begin{equation}
\min_{\XT} \frac{1}{2}\| \PO(\XT - \RT) \|_F^2  \quad \sbt \quad \XT \in \FM_r,
\label{eqn:originalRModel}
\end{equation}
with $\PO$ maps $\XT$ to the sparsified tensor $\PO(\XT)$, where
$\PO(\XT)(i_1,i_2,i_3) = \XT(i_1,i_2,i_3)$ if $(i_1,i_2,i_3)\in \Omega$,
and $\PO(\XT)(i_1,i_2,i_3) = 0$ otherwise.

Another popular model, Tucker model, is based on the quotient manifold $\MMr/\sim$, which can be expressed as:
\begin{equation}
\min_{\XT} \frac{1}{2}\| \PO( \rho^{-1}([\XT]) - \RT) \|_F^2  \quad \sbt \quad [\XT] \in \MMr / \sim,
\label{eqn:TuckerModel}
\end{equation}
with $\rho$ defined in Prop.~\ref{prop:diffeom}.

Note that since the dawn of Riemannian framework for LRTC, a quandary exists: on one hand, sparse measurement limits the capacity of the solution; on the other hand, rich side information can not be incorporated into this framework. In many artificial intelligence applications, demands for high accuracy further exacerbates such dilemma.

\begin{comment}
Various Euclidean models have been proposed to incorporate the side information. Most of these models
are base on tensor decomposition that a tensor can be decomposed into $\XT = \HT \times \VB_i $ where
$\VB_i$ is known as the mode-$i$ latent factor matrix. Among them the most typical are the 
Laplace regularized model and Co-factorization model. The Laplace model assumes that similar objects
should have similar latent factors and can be expressed as:
\begin{equation}
\min_{\HT,\VB_i} L(\HT,\{\VB_i\}_{i=1}^3) + \sum_{i=1}^3 \frac{\alpha_i}{2} \tr(\VB_i^\top \LB_i \VB_i)
\label{opt:laplace_model}
\end{equation}
where $L(\cdot)$ is the loss over the observed entries and $\LB_i$ is the Laplace matrix induced by the similar matrix $\AB_i$.
The Co-factorization model assumes that the latent factor matrix and feature matrix has the relation that
$\PB_i\WB_i = \VB_i\tilde{\WB}_i$ for some weight matrix $\WB_i, \tilde{\WB}_i$. And the model can be expressed as
\begin{equation*}
\min_{\HT,\VB_i, \WB_i,\tilde{\WB}_i} L(\HT,\{\VB_i\}_{i=1}^3) + \sum_{i=1}^3 \frac{\alpha_i}{2} \| \PB_i\WB_i - \VB_i\tilde{\WB}_i \|_F^2
\end{equation*}
\end{comment}

\section{Riemannian Model with Side Information}
% Assumptions
We focus on the case that the side information is encoded in feature matrices $\PB_i \in \RBB^{n_i\times k_i}$.
Suppose $\RT \in \FM_r$ has tucker factors $(\GT, \{\UB_i\}_{i=1}^3)$.Without loss of generality, we assume that $k_i \geq r_i$ and $\PB_i$ has orthogonal columns.

In the ideal case, we assume that 
\begin{equation}
\spann(\UB_i) \subset \spann(\PB_i).
\label{eqn:linearSpace}
\end{equation}
Such relation means that the feature matrices contain all the information in the latent space of the underlying tensor.
Equivalently, there exists a matrix $\WB_i$ such that $\UB_i = \PB_i \WB_i $.
However, in practice, due to the existence of noise, one can only expect such relation to hold approximately, i.e. $\UB_i \approx \PB_i \WB_i$. 
Incorporating such relation to a tensor completion model via penalization, we have the following formulation
\begin{equation}
\small
\label{eq:opt_over_total}
\begin{aligned}
&\min_{\GT, \UB_i, \WB_i} L(\GT,\{\UB\}_{i=1}^3) + \sum_{i=1}^3 \frac{\alpha_i |\Omega|}{2} \| \UB_i - \PB_i \WB_i \|_F^2,\\
&\quad \text{s.t.} \quad (\GT,\{\UB_i\}_{i=1}^3) \in \MMr,
\end{aligned}
\end{equation}
where $L(\GT,\{\UB_i\}_{i=1}^3) = \| \PO(\GT\times_{i=1}^3\UB_i-\RT) \|_F^2 / 2$.
Fixing $\GT$ and $\UB_i$, with respect to $\WB_i$, (\ref{eq:opt_over_total}) has a close form solution
\begin{equation}
\WB_i = (\PB_i^\top\PB_i)^{-1}\PB_i^\top \UB_i = \PB_i^\top \UB_i.
\label{eqn:express_of_weight}
\end{equation}
Since $\min_{x,y} l(x,y) = \min_{x} l(x,y(x))$ where $y(x) = \arg\min_{y} l(x,y)$,
one can substitute~(\ref{eqn:express_of_weight}) into the above problem and obtain the following equivalence
\begin{equation}\small
\begin{aligned}
&\min_{\GT, \UB_i}  L(\GT,\{\UB\}_{i=1}^3) + \sum_{i=1}^3 \frac{\alpha_i|\Omega|}{2} \tr(\UB_i^T(\IB_i - \PB_i\PB_i^T)\UB_i) \\
&  \quad \triangleq f(\GT,\{\UB_i\}_{i=1}^3)
\\
&\quad \text{s.t.} \quad (\GT,\{\UB_i\}_{i=1}^3) \in \MMr.
\label{eqn:lifted_optimization}
\end{aligned}
\end{equation}
Although the cost function is already smooth over the total space $\MMr$, due to its invariance over the equivalent class $[\GT, \{\UB_i\}_{i=1}^3]$, there can be infinite local optima, which is extremely undesirable.
%The cost function $f(\cdot)$ is smooth over the total space $\MMr$, and hence can be solved by Riemannian optimization methods \clr{\cite{absil2009optimization}}. 
%However, the cost is invariant on the equivalent class $[\GT, \{\UB_i\}_{i=1}^3]$, leading to infinite local optima in the total space, which is undesirable.
Indeed, if $(\GT, \{\UB_i\}_{i=1}^3)$ is a local optimal of the objective, then so is every point in the infinite set $[\GT,\{\UB_i\}_{i=1}^3]$. 
{One way to reduce the number of local optima is to mathematically treated the entire set $[\GT,\{\UB_i\}_{i=1}^3]$
	as a point. }
Consequently, we redefine the cost by $\tilde{f}([\GT,\{\UB_i\}_{i=1}^3]) = f(\GT,\{\UB_i\}_{i=1}^3)$ and obtain the following Remainnian optimization problem over the quotient manifold $\MMr/\sim$:
\begin{equation}
\min_{[\XT]} \tilde{f}([\XT]) \quad \text{s.t.}\quad [\XT] \in \MMr / \sim.
\label{eqn:optimization}
\end{equation}

\begin{remark}
	In Riemannian optimization literature, problem~(\ref{eqn:lifted_optimization}) is called the lifted representation of problem~(\ref{eqn:optimization}) over the total space \clr{\cite{absil2009optimization}}.  This model is closely related to the Laplace regularization model \clr{\cite{narita2011tensor}}. Concretely, they share the same form:
		\begin{equation}
		\min_{\GT,\UB_i} L(\GT,\{\UB_i\}_{i=1}^3) + \sum_{i=1}^3 \frac{C_i}{2} \tr(\UB_i^\top \LB_i \UB_i).
		\label{opt:general_model}
		\end{equation}
	The difference lies in that $\LB_i$ is a projection matrix in our case, while, in the Laplace regularization model, $\LB_i$ is a Laplacian matrix.
\end{remark}
\begin{remark}
	Since each $[\XT] \in \MM_r /\sim$ has a unique tensor representation in $\XT \in \FM_r$,
	we show that the abstract model~(\ref{eqn:optimization}) can be represented as a concrete model over the manifold $\FM_r$.
	%	it is reasonable to guess that the abstract model~(\ref{eqn:optimization}) can be represented as a concrete model over the manifold $\FM_r$. 
	Specifically, the following Proposition interprets the proposed model as an optimization problem with a regularizer that encourages the mode-$i$ space of the estimated tensor close to $\spann(\PB_i)$.
	\begin{proposition}
		\begin{small}
			if $[\XT]$ is a critical point of problem~(\ref{eqn:optimization}) then its tensor representation $\XT$ is
			a critical point of the following problem.
				\begin{equation*}
				\begin{aligned}
				\min_{\XT \in \FM_r} \frac{1}{2} \|\PO(\XT\!-\!\RT)\|_F^2 
				\!+\!  \sum_{i = 1}^3 \frac{\alpha_i |\Omega|}{2} {\dist}^2(\spann(\XT_{(i)}),\spann(\PB_i))	
				\end{aligned}
				\end{equation*}
			where $\dist(\cdot,\cdot)$ is the Chodal distance~\clr{\cite{ye2014distance}} between two subspaces. And vice versa. 
		\end{small}
		\label{prop:equiv_opt}
	\end{proposition}	
\end{remark}

\section{Riemannian Conjugate Gradient Descent}
%\subsection{ The Proposed Solver }
\begin{figure}
	\includegraphics[width = \columnwidth]{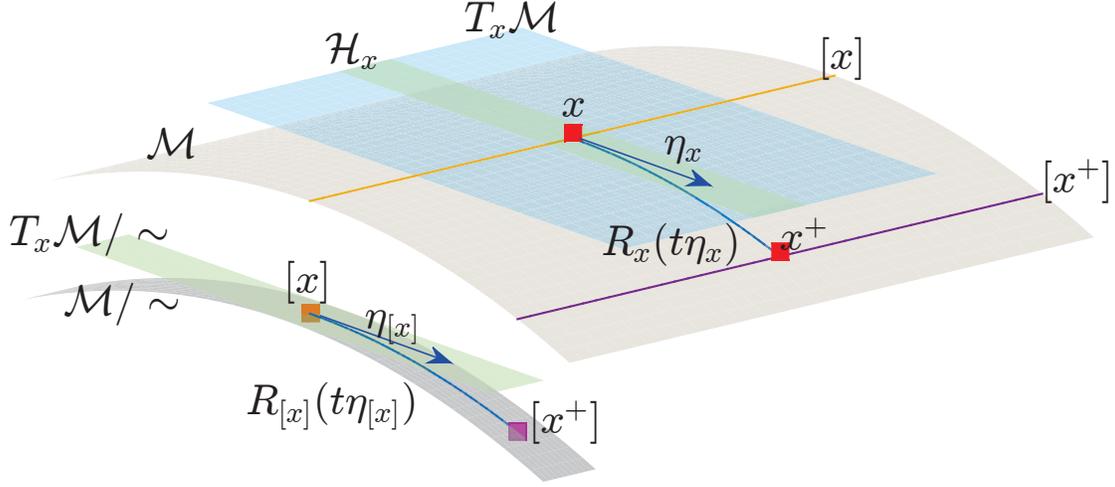}
	\caption{ Optimization Framework for Quotient Manifold: most Riemannian solvers  are based on the iteration formula:
		$[x^+] \leftarrow R_{[x]}(t\eta_{[x]})$, where $t > 0$ is the stepsize, $\eta_{[x]}$ is the search direction picked from current tangent space $T_{[x]}\MM / \sim$, and $R_{[x]}(\cdot)$ is the retraction, i.e. a map from current tangent space to $\MM/\sim$.
		Due to the abstractness of quotient manifold, such iteration is often lifted to (represented in) the total space as
		$x = R_{x}(t \eta_{x})$ where $x \in [x]$, $\eta_{x} $ is the horizontal lift of $\eta_{[x]}$, and $R_{x}(\cdot)$ is the lifted retraction. 
		Such representation is possible only if $\MM / \sim$ has the structure of Riemannian quotient, that is the total space
		is endowed with an invariant Riemannian metric.}
	\label{fig:quotient_optimization}
\end{figure}

% General optimization framework over quotient manifold
%The proposed problem~(\ref{eqn:optimization}) is an smooth optimization over quotient manifold~\cite{absil2009optimization}.
We depict the optimization framework for quotient manifolds in Fig.~\ref{fig:quotient_optimization}.
Under this framework, we solve the proposed problem (\ref{eqn:optimization}) by Riemannian Conjugate Gradient descent (CG).
% RCGD for our problem
With the details specified later, we list our CG solver for problem~(\ref{eqn:optimization}) in Alg.~\ref{alg:main},
where the CG direction is composed in the Polak-Ribiere+ manner with the momentum weight $\beta^{(k)}$ 
computed by Flecher-Reeves formula~\clr{\cite{absil2009optimization}}, and $\TM_{k}(\cdot)$ is the projector of horizontal
space  $\HM_{\overline{\XT^{(k)}}}$.
To represent Alg.~\ref{alg:main} in concrete tensor formulations, four items must be specified:
the Riemannian metric $\langle\cdot,\cdot\rangle_{\OXT}$, the Riemannian gradient $\grad f(\OXT)$,
the retraction $R_{\OXT}(\cdot)$, and the projector onto horizontal space $\TM_{\OXT}$.
% Amongst these, the Riemannian metric is most important, since it determines whether 
%$\MM / \sim$ has the structure of Riemannian quotient through which the CG solver can be represented
%in total space.

%---------------------------------solver-----------------------------------------------------------------------------------------
\begin{algorithm}
	\caption{CGSI: a Riemannian CG method}
	\label{alg:main}
	\begin{small}
		\begin{algorithmic}[1]
			\REQUIRE Initializer $\overline{\XT^{(0)}} = (\GT^{(0)},\{\UB^{(0)}_i\}_{i=1}^3)$ and tolerance $\epsilon$
			\STATE $k = 0$;
			\STATE $\eta^{(-1)} = (\zeroB, \{\zeroB\}_{i=1}^3)$;
			\REPEAT
			\STATE compute current Riemannian gradient $ \xi^{(k)} = \grad f(\XT^{(k)})$;
			\STATE compose CG direction $\eta^{(k)} = -\xi^{(k)} + \beta^{(k)} {\TM_{k} ( {\eta^{(k-1)}})}$;
			\STATE choose a step size $t_k > 0$;
			\STATE update by retraction $\overline{\XT^{(k+1)}} = R_{\OXTk}(t_k{\eta^{(k)}})$;
			\STATE $k = k + 1$;
			\UNTIL{$\langle \xi^{(k-1)}, \xi^{(k-1)} \rangle_{\overline{\XT^{(k-1)}}} \leq \epsilon$};
			\RETURN{ $\overline{\XT^{(k)}}$ }.
		\end{algorithmic}
	\end{small}
\end{algorithm}
%------------------------------------------------------------------------------------------------------------------------------------
% introduce metric

\subsection{ Metric Tuning }
Riemannian metric $\langle \cdot,\cdot\rangle_{\OXT}$ of $\MMr$ is an inner product defined over each tangent space $T_{\OXT}\MMr$. 
A high-quality Riemannian solver for a quotient manifold should be equipped with a well-tuned metric, because (1) the metric determines the differential structure of the quotient manifold, and more importantly (2) it implicitly endows the solver with a preconditioner, which heavily affects the convergent rate~\clr{\cite{mishra2014r3mc,mishra2014riemannian}}.

%When a highly accurate solution is demanded, 
%Thus, metric tuning should not only focus on the structure of the quotient manifold, but also consider the curvature of the cost function.

From the perspective of preconditioning, it seems that the best candidate is the Newton metric $\langle \eta, \xi\rangle_{\OXT} = D^2 f(\OXT)[\eta,\xi] \forall \eta, \xi \in T_{\OXT}\MMr$ where $D^2 f(\OXT)$ is the second order differential of the cost function. However, under such metric, computing the search direction involves solving a large system of linear equations, which precludes the Newton metric from the application to huge datasets. Therefore, we propose to use the following alternative:
\begin{equation}\small
\begin{aligned}
\small
\langle \etaOX, \xiOX \rangle_{\OXT} &= D^2 g(\OXT) [\etaOX, \xiOX] \\
&= \sum_{i = 1}^3 \langle \etaUi, \xiUi\GT_{(i)}\GT_{(i)}^\top \rangle +  \langle \etaG, \xiG \rangle \\
&+ \sum_{i = 1}^3 N\alpha_i  \langle \etaUi, (\IB_i - \PB_i\PB_i^\top)\xiUi \rangle,
\end{aligned}
\label{eqn:rMetric}
\end{equation}
in which $g(\OXT)$ is a scaled approximation to the original cost function, that is 
$
g(\OXT) \triangleq \frac{1}{2} \| \GT \times_1^3 \UB_i-\RT\|_F^2 +  \sum_{i = 1}^3 \frac{\alpha_iN}{2}   \tr(\UB_i^T(\IB_i - \PB_i\PB_i^T)\UB_i) 
$
with $N = n_1n_2n_3$. 

Our metric is more scalable than Newton metric. The following Proposition indicates that the scale gradient induced by this metric can be computed with only $O(\sum_{i=1}^3 n_i k_i r_i + r_i^3)$ additional operations, which is much less than the operations required by Newton metric.
\begin{proposition}\small
	Suppose that the cost function $f(\cdot)$ has Euclidean gradient $\nabla f(\OXT) = (\nabla_\GT f,\{ \nabla_{\UB_i} f\}_{i=1}^3)$.
	Then its scaled gradient $\tilde{\nabla} f(\OXT)$ under the metric~(\ref{eqn:rMetric}) can be computed by:
	%\begin{equation}
	\begin{align}
	\tilde{\nabla}_\GT f(\OXT) &= \nabla_\GT f(\OXT)， \nonumber\\
	\tilde{\nabla}_{\UB_i} f(\OXT) &= \EB_i \GB_i^{-1} + \FB_i (\GB_i + N\alpha_i\IB_i)^{-1}， \nonumber
	\end{align}
	%\end{equation}
	where $\EB_i = \PB_i\PB_i^\top\nabla_{\UB_i} f$, $\FB_i = \nabla_{\UB_i} f - \EB_i$, and $\GB_i = \GTi\GTi^\top$.
	\label{prop:scale_gradient}
\end{proposition}

Moreover, the proposed metric contains the curvature information of the cost. It is easy to validate that $D^2 f(\OXT) / |\Omega| \approx D^2 g(\OXT) / N$, since $f(\OXT) / |\Omega| \approx g(\OXT) / N$
if the observed entries are sampled uniformly at random.

The final proposition suggests that the proposed metric makes the representation of solvers in the total space possible. 
\begin{proposition}
	\begin{small}
		The quotient manifold $\MMr/\sim$ admits a structure of Riemannian quotient manifold,
		if $\MMr$ is endowed with the Riemannian metric defined in~(\ref{eqn:rMetric}).
	\end{small}
	\label{prop:metric}
\end{proposition}

\begin{table}
	\centering
	\begin{scriptsize}
		\begin{tabular}{|c|c|}
			\hline
			Projector & Formulation \\
			\hline
			\hline
			% Projector of tangent space
			\hline
			\makecell{
				$\Psi_{\OXT}(\ZB_\GT,\{\ZB_i\}_{i=1}^3)$\\
				projection of\\
				an ambient vector\\
				$(\ZB_\GT,\{\ZB_i\}_{i=1}^3)$\\
				onto
				$T_{\OXT} \MMr$\\
			} &
			\makecell{
				$
				\begin{aligned}
				(\ZB_\GT, \{\ZB_i - \VB_i\SB_i\GB_i^{-1}- \WB_i\SB_i\GB_{\alpha_i}^{-1}\}_{i=1}^3 )
				\end{aligned}$\\
				where
				%$\VB_i = \PB_i\PB_i^\top \UB_i$,$\WB_i = \UB_i - \VB_i$ \\
				%and
				$\SB_i$ is the solution of :\\
				$
				\begin{cases}
				\begin{aligned}
				&	\symm(\VB_i^T\VB_i \SB_i \GB_i^{-1} - \UB_i^\top\ZB_{i})\\
				&    \,\,\,\,\,= - \symm(\WB_i^T\WB_i\SB_i
				\GB_{\alpha_i}^{-1})  \\
				\end{aligned}\\
				\SB_i = \SB_i^\top
				\end{cases}
				$
			}\\
			%        see lemma ?\\
			% Projector of Horizontal space
			\hline
			\makecell{
				$\Pi_{\OXT}(\eta_{\OXT})$\\
				Projection of a\\
				tangent vector $\eta_\OXT$\\
				of total space\\
				onto $\HX$\\				
			}
			&
			\makecell{
				$
				\begin{aligned}
				(\etaG + \sum_{1\leq i \leq 3} \GT \times_i \OmegaB_i,
				\{\eta_{i} - \UB_i\OmegaB_i\}_{i=1}^3 )
				\end{aligned}
				$ \\
				where $(\OmegaB_1,\OmegaB_2,\OmegaB_3)$ is the solution of \\
				$
				\begin{cases}
				&\skewm(\VB_i^T\VB_i\OmegaB_i\GB_{i}+ \GB_i\OmegaB_i \\
				& \quad + \WB_i^\top\WB_i\OmegaB_i\GB_{\alpha_i}) \\
				& \quad - \GTi(\IB_{j_i}\otimes\OmegaB_{k_i} )\GTi^\top \\
				& \quad - \GTi(\OmegaB_{j_i} \otimes \IB_{k_i})\GTi^\top\\
				& = \skewm(\VB_i^\top\eta_i\GB_i + \WB_i^\top\eta_i\GB_{\alpha_i}) \\
				&	\quad  +  \skewm(\GTi(\etaG)_{(i)}^\top)  \\
				\\
				& \OmegaB_i^\top = - \OmegaB_i \forall i \in \{1,2,3\}
				\end{cases}
				$
			}\\
			\hline
		\end{tabular}
	\end{scriptsize}
	%	\captionsetup{font=scriptsize}
	\caption{Expressions of Projectors. 
		We define the following matrices: $\VB_i := \PB_i\PB_i^\top\UB_i,
		\WB_i := \UB_i - \VB_i$,
		$\GB_i := \GTi\GTi^\top$, $\GB_{\alpha_i} := N\alpha_i \IB_i + \GTi\GTi^\top$. $j_i  = \max\{k|k\in\{1,2,3\}, k\neq i\}$
		and $k_i = \min\{k|k\in\{1,2,3\}, k\neq i\}$. And the operator
		$\symm(\cdot)$ and $\skewm(\cdot)$ extract the symmetric and skew components
		of a matrix respectively, i.e. $\symm(\AB) = (\AB + \AB^\top)/2$ and $\skewm(\AB)= (\AB-\AB^\top)/2$.
		Note that the above linear systems can be solved by MATLAB command $pcg$ in
		$O(\sum_{1\leq i\leq 3} (n_i k_i^2 + r_i^3))$ flops.}
	\label{tab:GeometricObjects}
\end{table}

\subsection{Other Optimization Related Items}
\noindent \textbf{Projectors:} 
To derive the optimization related items, two orthogonal projectors, $\Psi_{\OXT}(\cdot)$ and $\Psi_{\OXT}(\cdot)$, are required.
The former projects a vector onto the tangent space $T_{\OXT} \MMr$, and the latter is a projector from the tangent space onto the horizontal space $\HM_{\OXT}$. The orthogonality of both projectors is measured by the metric~(\ref{eqn:rMetric}).
Mathematical derivation of these projectors are given in Sec.~\ref{sec:appendix:projector_hor} and Sec.~\ref{sec:appendix:tangent_space}.

\noindent\textbf{Riemannian Gradient:} According to~\clr{\cite{absil2009optimization}}, the Riemannian gradient
can be computed by projecting the scaled gradient onto tangent space, specifically
\begin{equation}
\grad f(\OXT) = \Psi_{\OXT}(\tilde{\nabla} f(\OXT)).
\end{equation}

\noindent\textbf{Retraction:} We use the retraction defined by 
\begin{equation}
R_{\OXT}(\eta_{\OXT}) = (\GT + \etaG, \{\uf(\UB_i+\eta_{i})\}_{i=1}^3 ). 
\label{eqn:retraction}
\end{equation}
where $\uf(\cdot)$ extracts the orthogonal component from a matrix.
Such retraction is proposed by~\clr{\cite{kasai2016low}}. we give rigorous analysis to prove that the above retraction is compatible with the proposed metric in Sec.~\ref{sec:appendix:retraction}.
% which means  $R_{\OXT}(\cdot)$ can represent a retraction of the quotient manifold under the Riemannian quotient structure induced by~(\ref{eqn:rMetric}).

\section{Experiments}

We validate the effectiveness of the proposed solver CGSI by comparing it with the state-of-the-art.
The baseline can be partitioned into three classes. The first class contains Riemannian solvers including GeomCG \clr{\cite{kressner2014low}}, FTC \clr{\cite{kasai2016low}}, and gHOI \clr{\cite{liu2016generalized}}. The second class consists of Euclidean solvers that take no account of the side information, including AltMin \clr{\cite{romera2013multilinear}} and HalRTC \clr{\cite{liu2013tensor}}.
The third class comprises of two methods that incorporate side information, including RUBIK \clr{\cite{wang2015rubik}} and TFAI \clr{\cite{narita2011tensor}}.
All the experiments are performed
in Matlab on the same machine with 3.0 GHz Intel E5-2690 CPU and 128GB RAM.

All solvers are based on the Tucker decomposition, except that RUBIK is based on the CP decomposition.
For fairness, the CP rank of RUBIK is set to $\lceil (\sum_{i=1}^3 n_i r_i + r_1r_2r_3)  / (\sum_{i=1}^3 n_i)\rceil$.

\subsection{Hyperspectral Image Inpainting}
% what is hyperspectral image
A hyperspectral image is a tensor whose the slices are photographs of the same scene under different wavelets.
We adopt the dataset provided in \clr{\cite{foster2006frequency}} which contains images about eight different rural scenes taken under 33 various
wavelets.
To make all methods in our baseline applicable to the completion problem, we resize each hyperspectral images to a small dimension such that $n_1 = 306$, $n_2 = 402$, and $n_3=33$. 
% extract samples
Empirically, we treat these graphs as tensors of rank $r = (30,30,6)$.
The observed pixels, or the training set, are sampled from the tensors uniformly at random. And the sample size is set 
to $|\Omega| = OS \times p$ in which $OS$ is so-called Over-Sampling ratio and $p = \sum_{i=1}^3 (n_ir_i - r_i^2) + r_1r_2r_3$ is the number of free parameters in a size $n$ tensor with rank $r$. 
% how to construct side infomation
In addition to the observed entries, the mode-$1$ feature matrix is constructed by extracting the top-$(r_1 + 10)$ singular vectors
from a matrix of size $n_1 \times 10r_1$ whose columns are sampled from the mode-1 fibers of the hyperspectral graphs.
% evaluation metric
The recovery accuracy is measured by Normalized Root Mean Square Error (NRMSE)~\clr{\cite{kressner2014low}}. All the compared methods are
terminated when the training NRMSE is less
than $0.003$ or iterate more than $300$ epochs. We report the NRMSE and CPU time of the compared methods in Tab. \ref{tab:hyperGraph}. From the
table, we can see that the proposed method has much higher accuracy than the other solvers in our baseline. 
The empirical results also indicate that our method has nearly the same running time with FTC, the fastest tensor completion method.
The visual results of the 27th slices of recovered hyperspectral images of scene 7 are illustrated in
Fig.~\ref{fig:scene7}.

\begin{sidewaystable}
	\caption{Performance of the compared methods on hyperspectral images.}
	\begin{tiny}
		\begin{tabular}{|c|c|c|c|c|c|c|c|c|c|c|c|c|c|c|c|c|c|c|c| }
			\hline
			& & \multicolumn{2}{c|}{AltMin} & \multicolumn{2}{c|}{FTC} & \multicolumn{2}{c|}{GeomCG} & \multicolumn{2}{c|}{gHOI}
			& \multicolumn{2}{c|}{HalRTC} & \multicolumn{2}{c|}{RUBIK} & \multicolumn{2}{c|}{TFAI} & \multicolumn{2}{c|}{CGSI} \\
			\hline
			data & OS & NRMSE & Time(s) & NRMSE & Time & NRMSE & Time & NRMSE & Time & NRMSE & Time & NRMSE & Time 
			& NRMSE & Time & NRMSE & Time \\
			\hline
			\multirow{5}{5mm}{Scene1} 
			% altmin & ftc & geomcg & ghoi & halrtc & rubic
			%& 1 % 61 	40 	53 	42 	174 	187 	56 	60
			%& 0.163 & 61 & 0.1221 & 40 & 0.133 & 53 &  0.137 & 42 & 0.152 & 174 & 0.096 & 187 & 0.163 & 56 & \textbf{0.083} & 60 \\
			&3 % 183 	52	61 	65 	177	197 	164 	77
			& 0.161 & 183 & 0.091 & \textbf{52} & 0.113 & 61 &  0.115 & 65 & 0.080 & 177 & 0.086 & 197 & 0.161 & 164 & \textbf{0.062} & 77 
			\\ 
			&5 % 307 	76 	93 	109 	177 	194 	273 	100
			& 0.156 & 307 & 0.067 & \textbf{76} & 0.077 & 93 &  0.103 & 109 & 0.078 & 177 & 0.085 & 194 & 0.159 & 273 & \textbf{0.040} & 100 
			\\ 
			&7 % 429 	100 	124 	152 	177 	195 	382 	110
			& 0.156 & 429 & 0.060 & \textbf{100} & 0.056 & 124 &  0.092 & 152 & 0.077 & 177 & 0.085 & 195 & 0.159 & 382 & \textbf{0.039} & 110 
			\\ 
			&9 % 550	126 	151 	195 	178 	198 	479 	126
			& 0.156 & 550 & 0.046 & \textbf{126} & 0.044 & 151 &  0.078 & 195 & 0.077 & 178 & 0.085 & 198 & 0.156 & 479 & \textbf{0.036} & 126 
			\\
			\hline
			\hline
			\multirow{5}{5mm}{Scene2} 
			%& 1  % 61 	26 	30 	27 	173 	197 	56	65
			%& 0.174 & 61 & 0.101 & \textbf{26} & 0.127 & 30 &  0.174 & 27 & 0.142 & 173 & 0.064 & 197 & 0.175 & 55 & \textbf{0.078} & 65 \\
			&3 & % 183 	50 	61 	65 	173 	197	165 	83
			0.173 & 183 & 0.093 & \textbf{50} & 0.114 & 61 &  0.125 & 65 & 0.066 & 173 & 0.061 & 197 & 0.173 & 165 & \textbf{0.048} & 83 \\ 
			&5 & % 306 	76 	92	103 	171 	196 	203 	96
			0.166 & 306 & 0.082 & \textbf{76} & 0.076 & 92 &  0.100 & 103 & 0.066 & 171 & 0.061 & 196 & 0.171 & 203 & \textbf{0.043} & 96 \\ 
			&7 & % 428 	101 	123 	152 	172 	197 	386 	110	
			0.166 & 428 & 0.073 & \textbf{101} & 0.064 & 123 &  0.091 & 152 & 0.057 & 172 & 0.061 & 197 & 0.169 & 386 & \textbf{0.040} & 110 \\ 
			&9 & % 	578 	125 	154 	197 	171 	197 	433 	130
			0.166 & 578 & 0.062 & \textbf{125} & 0.056 & 154 &  0.073 & 197 & 0.057 & 171 & 0.060 & 197 & 0.169 & 433 & \textbf{0.038} & 130 \\ 
			\hline
			\hline
			\multirow{5}{5mm}{Scene3} 
			%& 1 		
			%& 0.031 & 90 & 0.045  & 38 & 0.050 & 102 &  0.040 & 59 & 0.050  & 167 &  0.057 & 193 & 0.057 & 33 & \textbf{0.036} & 84 \\
			&3 & 
			0.033 & 226  & 0.041   &\textbf{68}  &0.044  &181    &  0.043 & 187  &0.034  &174   &  0.062  & 189  & 0.063 & 131  & \textbf{0.025} & 83 \\
			&5 &
			0.033 & 346  & 0.030   &\textbf{99}  &0.029  &251   &  0.037 & 308  &0.033  &177   &  0.061  & 185 & 0.062 & 209  & \textbf{0.021} & 108 \\ 
			&7 & 
			0.033 & 486  & 0.023  &\textbf{124}  &0.021  & 389   &  0.033 & 177  &0.031  &177   &  0.059  & 187  & 0.057 & 210 & \textbf{0.018} & 131 \\ 
			&9 & 
			0.033 & 587  & 0.019  &\textbf{156}  &0.021  & 386    &  0.031 & 491  &0.029 & 172   &  0.034  & 189  & 0.033 & 229  & \textbf{0.017} & 143 \\ 
			\hline
			\hline
			\multirow{5}{5mm}{Scene4} 
			%& 1 
			% 
			%& 0.033 & 58 & 0.048 & 40 & 0.055 & 48 &  0.053 & 49 & 0.067 & 168 & 0.043 & 181 & 0.034 & 53 & \textbf{0.024} & 88 \\
			&3 & 
			0.033 & 238 & 0.031 & \textbf{78} & 0.036 & 181 &  0.038 & 193 & 0.047 & 172 & 0.034 & 182 & 0.033 & 155 & \textbf{0.012} & 105 \\ 
			&5 & 
			0.033 & 359 & 0.015 & \textbf{108} & 0.015 & 254 &  0.031 & 293 & 0.032 & 171 & 0.037 & 183 & 0.033 & 247 & \textbf{0.012} & 118 \\ 
			&7 & 
			0.033 & 486 & 0.012 & \textbf{128} & 0.012 & 391 &  0.021 & 177 & 0.029 & 177 & 0.027 & 180 & 0.033 & 181 & \textbf{0.011} & 131 \\ 
			&9 & 
			0.033 & 600 & 0.012 & 170 & 0.012 & 398 &  0.018 & 492 & 0.026 & 177 & 0.024 & 192 & 0.033 & 231 & \textbf{0.010} & \textbf{144} \\ 
			\hline
			\hline
			\multirow{5}{5mm}{Scene5} 
			%& 1 & 
			%0.061 & 61 & 0.115 & 26 & 0.212 & 30 &  0.317 & 22 & 0.097 & 87 & 0.127 & 79 & 0.077 & 75 & \textbf{0.095} & 27 \\
			&3 & 
			0.059 & 236 & 0.051 & \textbf{75} & 0.077 & 180 &  0.169 & 187 & 0.086 & 169 & 0.126 & 180 & 0.062 & 99 & \textbf{0.024} & 104 \\ 
			&5 & 
			0.059 & 362 & 0.041 & \textbf{104} & 0.051 & 254 &  0.113 & 289 & 0.076 & 171 & 0.059 & 183 & 0.061 & 128 &\textbf{ 0.022} & 114 \\  
			&7 & 
			0.059 & 483 & 0.034 & 137 & 0.037 & 325 &  0.089 & 398 & 0.047 & 173 & 0.054 & 190 & 0.061 & 181 &\textbf{ 0.021} & \textbf{128} \\  
			&9 & 
			0.059 & 603 & 0.028 & 166 & 0.029 & 400 &  0.065 & 494 & 0.042 & 173 & 0.058 & 192 & 0.061 & 229 & \textbf{0.021} & \textbf{142} \\  
			\hline
			\hline				
			\multirow{5}{5mm}{Scene6} 
			%& 1  
			%& 0.092 & 89 & 0.150 & 39 & 0.213 & 101 &  0.193 & 64 & 0.141 & 173 & 0.183 & 110 & 0.091 & 98 & \textbf{0.083} & 89 \\
			&3 
			& 0.090 & 237 & 0.067 & \textbf{76} & 0.057 & 181 &  0.132 & 189 & 0.095 & 177 & 0.090 & 180 & 0.091 & 170 & \textbf{0.036} & 107 \\
			&5 & 
			0.090 & 356 & 0.039 & \textbf{105} & 0.040 & 251 &  0.095 & 298 & 0.083 & 177 & 0.081 & 180 & 0.091 & 213 & \textbf{0.034} & 119 \\ 
			&7 & 
			0.090 & 489 & 0.039 & \textbf{130} & 0.040 & 325 &  0.095 & 394 & 0.083 & 178 & 0.081 & 181 & 0.091 & 300 &\textbf{ 0.034} & 136 \\  
			&9 &  
			0.090 & 600 & 0.039 & 165 & 0.040 & 396 &  0.095 & 501 & 0.083 & 178 & 0.081 & 183 & 0.091 & 383 & \textbf{0.034} & \textbf{143} \\ 
			\hline
			\hline	
			\multirow{5}{5mm}{Scene7}
			% & 1 & % 96 	39 	125	67 	170 	180 	48 	94
			%0.072 & 96 & 0.110 & 39 & 0.134 & 125 &  0.089 & 67 & 0.119 & 170 & 0.114 & 180 & 0.073 & 48 & \textbf{0.052} & 94 \\
			&3 & %238	78	181	193	172	264	155	105
			0.071 & 245 & 0.073 & \textbf{82} & 0.069 & 181 &  0.075 & 193 & 0.077 & 172 & 0.069 & 181 & 0.072 & 165 & \textbf{0.031} & 119 \\ 
			&5 & % 0.072 	0.034 	0.032 	0.064 	0.069 	0.067 	0.072 	0.028
			0.072 & 377 & 0.034 & \textbf{102} & 0.032 & 225 &  0.064 & 293 & 0.069 & 172 & 0.067 & 180 & 0.072 & 203 & \textbf{0.028} & 158 \\ 
			&7 & 
			0.072 & 581 & 0.028 & 161 & 0.028 & 336 &  0.052 & 452 & 0.062 & 171 & 0.064 & 181 & 0.072 & 302 & \textbf{0.026} & \textbf{157} \\ 
			&9 & %
			0.072 & 603 & 0.027 & \textbf{170} & 0.027 & 400 &  0.041 & 494 & 0.057 & 173 & 0.058 & 183 & 0.072 & 183 & \textbf{0.026} & 189 \\
			\hline
			\hline	
			\multirow{5}{5mm}{Scene8} 
			%& 1 & 
			%0.040 & 92 & 0.057 & 41 & 0.077 & 96 &  0.070 & 72 & 0.075 & 171 & 0.043 & 176 & 0.041 & 48 & \textbf{0.023} & 87 \\
			&3 & 
			0.039 & 236 & 0.030 & \textbf{74} & 0.042 & 181 &  0.050 & 187 & 0.071 & 174 & 0.034 & 179 & 0.040 & 131 & \textbf{0.013} & 103 \\
			&5 & 
			0.039 & 354 & 0.018 & \textbf{107} & 0.019 & 247 &  0.038 & 293 & 0.061 & 174 & 0.040 & 182 & 0.045 & 213 & \textbf{0.012} & 114 \\
			&7 &
			0.039 & 701 & 0.013 & \textbf{102} & 0.013 & 381 &  0.030 & 234 & 0.031 & 181 & 0.030 & 182 & 0.060 & 363 & \textbf{0.011} & 169 \\ 
			&9 &
			0.039 & 853 & 0.012 & \textbf{112} & 0.012 & 502 &  0.026 & 369 & 0.027 & 175 & 0.031 & 183 & 0.039 & 502 & \textbf{0.011} & 180 \\ 
			\hline																		
		\end{tabular}
	\end{tiny}
	\label{tab:hyperGraph}
\end{sidewaystable}
\begin{figure*}[!thb]
	\centering	
	\begin{small}
		\begin{tabular}{cccccc}
			Original & Observed & CGSI & RUBIK & FTC  \\
			\includegraphics[width = .17\textwidth]{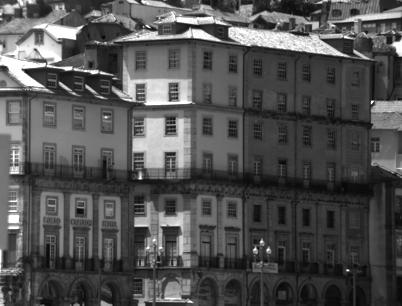} &
			\includegraphics[width = .17\textwidth]{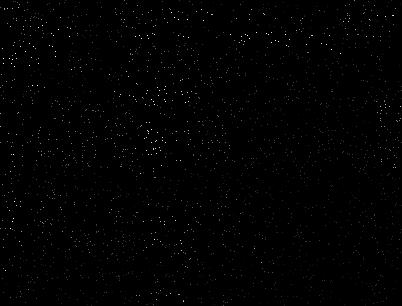} &
			\includegraphics[width = .17\textwidth]{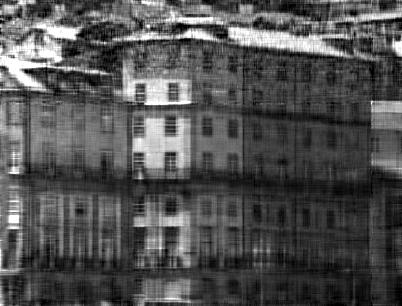} &
			\includegraphics[width = .17\textwidth]{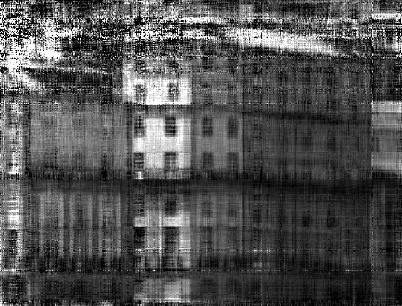} &
			\includegraphics[width = .17\textwidth]{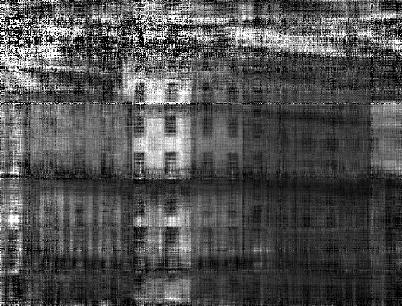} 
		\end{tabular}
	\end{small}
	\caption{Visual results of the recovered 27th frame of scene7 when OS is set to 3.}
	\label{fig:scene7}
\end{figure*}

\subsection{Recommender System}
In recommendation tasks, 
two datasets are considered: MovieLens 10M (ML10M) and MovieLens 20M (ML20M). Both datasets contain the rating history
of users for items at specific moments.
For both datasets, we partition the samples into 731 slices in terms of time stamp. Those slices have the identical time intervals. Accordingly, the completion tasks for the two datasets
are of sizes $71567\times10681\times731$ and $ 138493\times 26744 \times731$ respectively.
In addition to the rating history, both datasets contain two extra files: one describes the genres of each movie, and the other contains tags of each movie.
We construct a corpus that contains the text description of all movies from the genres descriptions and all the tags.
The feature matrix is extracted from the above corpus by the latent semantic analysis (LSA) method. The processing is efficient since LSA is implemented via randomized SVD.

Various empirical studies are conducted to validate the performance of the proposed method. In the first scenario, we record 
the CPU time and the Root Mean Square Error (RMSE) outputted by the compared algorithms under different choices of multi-linear rank. 
In this scenario, for both datasets, $80\%$ samples are chosen as training set, and the rest are left for testing. The results are listed in Tab.~\ref{tab:Rec_RMSE}, which suggests that the proposed method outperforms all other solvers in terms of accuracy. 
For ML10M, our method uses significantly less CPU time than its competitors.
In Fig.~\ref{fig:Rec_Accuarcy_train_ratio}, we report another scenario, in which the percentage of training samples are
varied from $10\%$ to $70\%$ and the rank parameter is fixed to $(10,10,10)$. Experimental results in this figure indicate that 
our method has the lowest RMSE.

To show the impact of parameter $\alpha$ on the performance of our method, we depict the relation between RMSE and $\alpha$ in Fig.~\ref{fig:RMSE_alpha},  
where the rank parameter is set to $(10,10,10)$, and the partitioning scheme for training and testing samples is the same as the first scenario. 
From this Figure we can see that our method has higher accuracy than the vanilla Riemannian model's solver FTC for a wide range of parameter choices.
\begin{figure}
	\includegraphics[width = 0.95 \columnwidth]{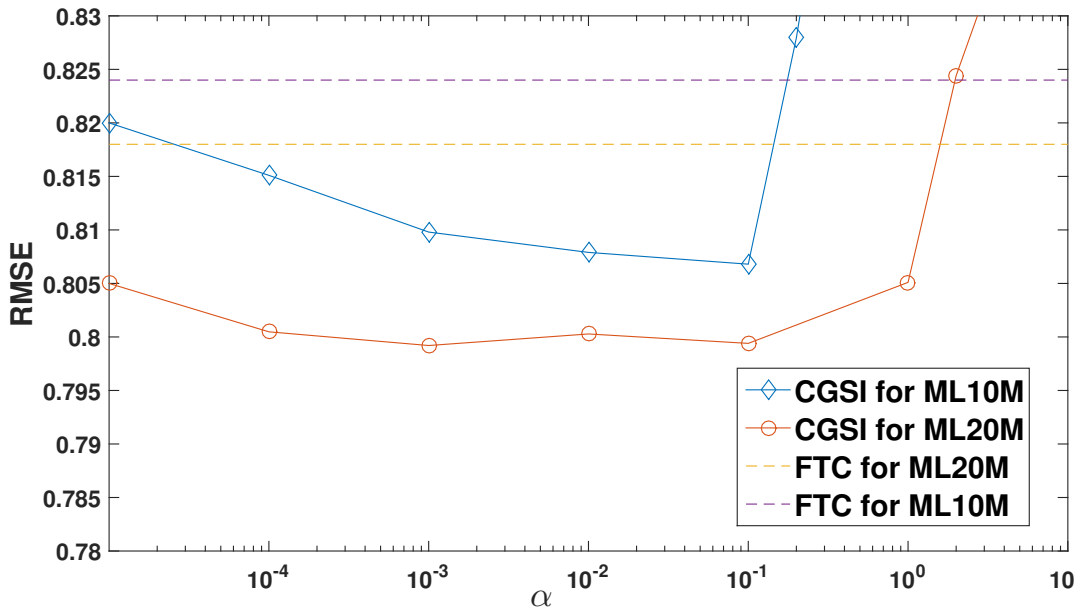}
	\caption{Effect of parameter $\alpha$ on the accuracy of CGSI.}
	\label{fig:RMSE_alpha}
\end{figure}

\begin{table*}[!th]
	\caption{Performance of the compared methods on Recommendation Tasks.}
	\label{tab:Rec_RMSE}
	\begin{tiny}
		\begin{tabular}{|c|c|c|c|c|c|c|c|c|c|c|c|c|c|c|c| }
			\hline
			& & \multicolumn{2}{c|}{AltMin} & \multicolumn{2}{c|}{FTC} & \multicolumn{2}{c|}{GeomCG} & \multicolumn{2}{c|}{gHOI}
			& \multicolumn{2}{c|}{TFAI} & \multicolumn{2}{c|}{CGSI} \\
			\hline
			dataset & rank & RMSE & Time(s) & RMSE & Time & RMSE & Time & RMSE & Time & RMSE & Time & RMSE & Time \\
			\hline
			\multirow{4}{10mm}{ML10M} & (4,4,4) 
			% 924,	236,	307,	467,	426,	178
			&  0.982 & 924  &0.824  &  236  &0.835 & 307   &	1.076 & 467 & 1.011 &426  &	\textbf{0.823} & \textbf{178} \\
			& (6,6,6) 
			% 1830,	535,	679,	1035,	942,	434
			&  0.968 & 1830  & \textbf{0.814}  &  535  &0.826 & 679   &	1.262 & 1035 & 0.9948 & 942 &	\textbf{0.814} & \textbf{434}\\
			& (8,8,8) 
			% 3123,	928,	1135,	1734,	1617,	754
			&  1.01 & 3123  &0.822  &  928  &0.833 &  1135  &	1.062 & 1734 & 0.993 & 1617 &	\textbf{0.810} & \textbf{754}\\	
			& (10,10,10) 
			% 4963,	1631,	2220,	2788,	2522,	1067
			&  1.147 &  4963 &0.824  & 1631   &0.843 & 2220   &	1.094 & 2788 & 0.992 & 2522 &	\textbf{0.807} & \textbf{1067} \\
			\hline	
			\hline	
			\multirow{4}{10mm}{ML20M} & (4,4,4) 
			% 690,	466,	601,	918,	797,	363
			&  1.061 & 690  &0.822  & 466   &0.829 & 601   &	1.050 &  918& 1.029 &  797&	\textbf{0.818} & \textbf{363}\\
			& (6,6,6) 
			% 3451,	982,	1309,	1869,	1644,	1107
			& 1.089 &  3451 & {0.808}  &  \textbf{982}  &0.822 &  1309  &	1.057 & 1869 & 1.008 & 1644 &	\textbf{0.805} & 1107\\
			& (8,8,8) 
			% 5890,	1725,	2271,	3363,	3144,	1739
			&  1.092 & 5890  &0.812  &  \textbf{1725}  &0.828 &  2271  &	1.045 &3363  & 1.004 & 3144 &	\textbf{0.804} &1739 \\	
			& (10,10,10) 
			% 9418,	3161,	4308,	5795,	5394,	2813
			&  1.092 & 9418  &0.818  & {3161}   &0.834 & 4308   &	1.054 & 5795 & 1.025 & 5394 &	\textbf{0.799} & \textbf{2813}\\		
			\hline					
		\end{tabular}
	\end{tiny}
\end{table*}

\begin{figure*}[!h]
	\centering
	\includegraphics[width = 1\textwidth]{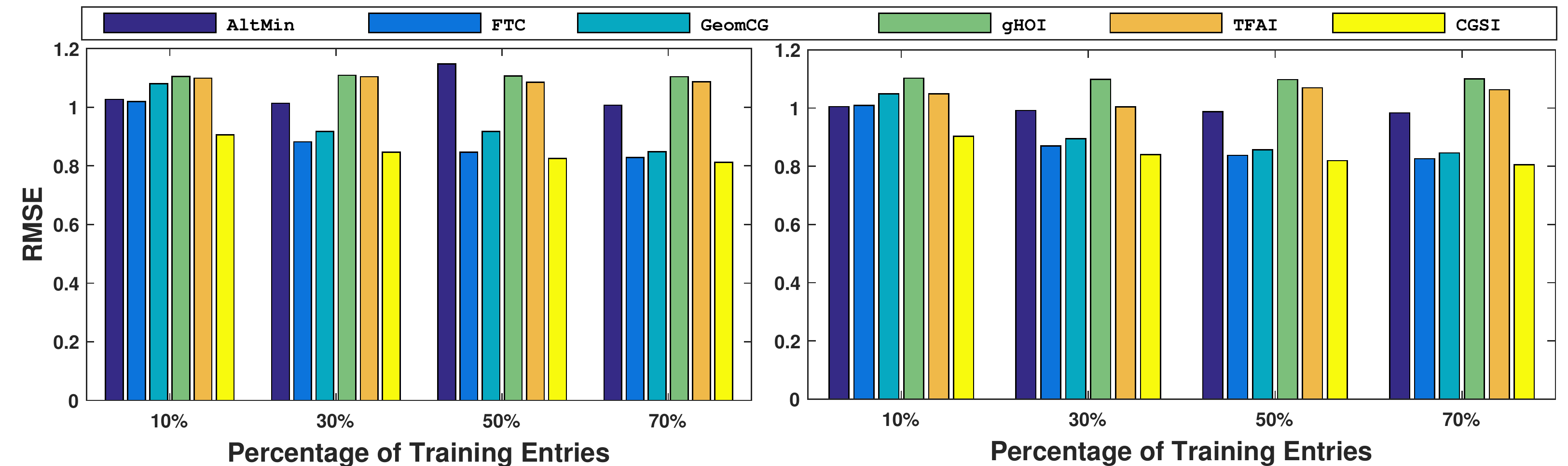}
	\caption{Accuracy of compared methods under different size of training set}
	\label{fig:Rec_Accuarcy_train_ratio}
\end{figure*}

\section{Conclusion}
In this paper, we exploit the side information to improve the accuarcy of Riemannian
tensor completion. A novel Riemanian model is proposed.
To solve the model efficiently, we design a new Riemannian metric. Such metric will induce an adaptive preconditioner for the
solvers of the proposed model. Then, we devise a Riemannian conjugate gradient descent
method using the adaptive preconditioner. Empirical results
show that our solver outperforms state-of-the-arts.

\section{appendix}
\subsection{ Proof of Propositions }
Before delve into the proofs of the propositions, we construct the submersion between the total space
$\MMr$ and fix multilinear rank manifold $\FM_r$ in the following Lemma.
\begin{lemma}
	Let $\pi: \MMr \rightarrow \FM_r$ be a mapping defined by 
	\[
	\pi(\GT,\UB_1,\UB_2,\UB_3) = \GT \times_{i=1}^3 \UB_i.
	\]
	Then it is a submersion from $\MMr$ to $\FM_r$.
	\label{lemma:submersion}
\end{lemma}
\begin{proof}
	To begin with, we define a function $\pi : \MMr \rightarrow \FM_r$ as follows
	\[
		\pi(\GT,\UB_1,\UB_2,\UB_3) = \GT \times_{i=1}^3 \UB_i.
	\]
	Note that $\pi()$ is a smooth function over $\MMr$, and for all $\OXT = (\GT,\UB_1,\UB_2,\UB_3) \in \MMr$,
	and for all the tangent vectors $\eta_\OXT = (\eta_\GT, \eta_1, \eta_2, \eta_3) \in T_\OXT \MMr$, the 
	first order derivative of $\pi()$ can be computed as follows:
	\begin{equation}
		D\pi(\OXT)[\eta_\OXT] = \eta_\GT \times_{i=1}^3 \UB_i + \GT \times_1 \eta_1 \times_2 \UB_2 \times_3 \UB_3
			+ \GT \times_1 \UB_1 \times_2 \eta_2 \times_3 \UB_3
			+\GT \times_1 \UB_1 \times_2 \UB_2 \times_3 \eta_3
		\label{eqn:derivative_of_pi_1}
	\end{equation}
	Note that $\eta_G \in \RBB^{r_1\times r_2 \times r_3}$ and $\eta_i \in T_{\UB_i}\st(r_i,n_i)$ which means
	they can be expressed as $\eta_i = \UB_i \OmegaB_i + \UB_{i,\perp}\KB_i$ where $\Omega_i \in \RBB^{r_i \times r_i}$
	is a skew matrix, $\KB_i \in \RBB^{(n_i-r_i)\times r_i}$ and
	$\UB_{i,\perp} \in \RBB^{n_i\times(n_i-r_i)}$ is the orthogonal basis, the spanned subspace of which is
	the orthogonal complement of $\spann(\UB_i)$. Substitute these expressions to equation~(\ref{eqn:derivative_of_pi_1}),
	we have:
	\begin{equation}
		\begin{aligned}
		 D\pi(\OXT)[\eta_\OXT] = (\eta_\GT + \sum_{i=1}^3 \GT \times_i \OmegaB_i) \times_{i=1}^3 \UB_i 
			+ \sum_{i=1}^3 \GT \times_i \UB_{i,\perp}\KB_i \times_{j\neq i, 1\leq j\leq3} \UB_j.
		\end{aligned}
		\label{eqn:derivative_of_pi_2}
	\end{equation}
	Therefore, the range of the map $ D\pi(\OXT)[\cdot]$ over the tangent space $T_\OXT \MMr$
	\begin{equation}
		\begin{aligned}
		\range(D\pi(\OXT)) = \left\{ \HT \times_{i=1}^3 \UB_i  
		+ \sum_{i=1}^3 \GT \times_i \UB_{i,\perp}\KB_i \times_{j\neq i, 1\leq j\leq3} \UB_j |
			\HT \in \RBB^{r_1\times r_2 \times r_3}, \KB_i \in \RBB^{(n_i-r_i)\times r_i} \right\}
		\end{aligned}
	\end{equation}
	Note that the tangent space of fix multilinear rank manifold $\FM_r$ at the point $\XT = \pi(\GT, \UB_1, \UB_2,\UB_3)$ is
	\begin{equation}
		T_{\XT}\FM_r = \left\{ \HT \times_{i=1}^3 \UB_i +  \sum_{i=1}^3 \GT \times_i \VB_i \times_{j\neq i, 1\leq j \leq 3} \UB_j |
			\HT \in \RBB^{r_1 \times r_2 \times r_3}, \VB_i \in \RBB^{n_i\times r_i}\text{ and } \VB_i\UB_i = \zeroB \right\}.
	\end{equation} 	
	Using the fact any matrix $\VB_i \in \RBB^{n_i \times r_i}$ and $\VB_i^\top \times \UB_i = \zeroB$, there exist $\KB_i \in \RBB^{(n_i-r_i)\times r_i}$
	such that $\VB_i = \UB_{i,\perp}\KB_i$, we can infer that
	\begin{equation}
		\range(D\pi(\OXT)) = T_{\XT}\FM_r.
	\end{equation}
	As a result, $\pi(\cdot)$ is a submersion from $\MMr$ to $\FM_r$.
	\end{proof}

\subsubsection{Horizontal Space}
\begin{proposition}
	Let $\OXT = (\GT,\UB_1,\UB_2,\UB_3) \in [\XT]$, the horizontal space of $\MMr$ at point $\OXT$
	is 
	\[
	\big\{\eta_{\OXT} \in T_{\OXT}\MMr | \VB_i^\top\eta_i\GB_i + \WB_i^\top\eta_i\GB_{\alpha_i}
	\text{is symmetric} \forall 1\leq i \leq 3 \big\}
	\]
	where  $\VB_i = \PB_i\PB_i^\top\UB_i$, $\WB_i = \UB_i - \PB_i\PB_i^\top\UB_i$,
	$\GB_i = \GT_{(i)}\GT_{(i)}^\top$, $\GB_{\alpha_i} = N\alpha_i\IB_i + \GT_{(i)}\GT_{(i)}^\top$.
	\label{Prop:horizontal_space}
\end{proposition}
\begin{proof}
	Let $\XT\in \FM_r$ be a tensor with tucker factorization $\OXT = (\GT,\UB_1,\UB_2,\UB_3) \in [\XT]$.
	In quotient manifold framework~\cite{absil2009optimization}, the
	equivalent class $[\XT]$ is called the fiber of total space.
	The lifted representation of the tangent
	space $T_{[\XT]}\MMr / \sim $ is identified with a subspace of the tangent space $T_{\OXT} \MMr$ that does not produce a displacement along the
	fiber $[\XT]$. This is realized by decomposing $T_{\OXT} \MMr$ into two complementary subspace, the vertical and horizontal
	spaces, such that $T_{\OXT} \MMr = \HX \oplus \VX$, where $\HX$ is the horizontal space and $\VX$ is the vertical space. 
	It should be emphasized that the decomposition is respect to the metric~(\ref{eqn:rMetric}).
	The vertical space $\VX$ is the tangent space of the fiber $[\XT]$. According to~\cite{kasai2016low},
	the vertical space can be expressed as follows.
	\begin{equation}
	\VX = \{(-\sum_{i=1}^3\GT\times_i\OmegaB_i, \UB_1\OmegaB_1,\UB_2\OmegaB_2,\UB_3\OmegaB_3)|
	\OmegaB_i^\top = -\OmegaB_i\}.
	\label{eqn:vertical_space}
	\end{equation}
	Since horizontal space
	$\HX$ is an orthogonal complement of $\VX$ with respect to the Riemannian metric~(\ref{eqn:rMetric}),
	for all horizontal vectors
	$\eta_{\OXT} = (\eta_\GT,\eta_1,\eta_2,\eta_3) \in \HX$ we have
	\begin{equation}
	\langle\eta_{\OXT},\zeta \rangle_{\OXT} = 0, \forall \zeta \in \VX.
	\label{eqn:horizonspaceEqn1}
	\end{equation}
	Using the expression for the horizontal space, the above equation is equivalent to
	the following one:
	\begin{equation}
	\begin{small}
	\begin{aligned}
	&\sum_{i = 1}^3 \langle \etaUi, \UB_i\OmegaB_i\GT_{(i)}\GT_{(i)}^\top \rangle
	+  \langle \etaG, -\sum_{i=1}^3\GT\times_i\OmegaB_i \rangle \\
	&\quad + \sum_{i = 1}^3 N\alpha_i  \langle \etaUi, (\IB_i - \PB_i\PB_i^\top)\UB_i\Omega_i \rangle = 0.	\end{aligned}
	\end{small}
	\label{eqn:horizonspaceEqn2}
	\end{equation}
	Using the property for the Euclidean inner product that for matrix $\AB,\BB,\CB,\DB$
	we have $\langle \AB,\BB\CB\DB\rangle = \langle\BB^\top\AB\DB^\top,\CB \rangle$. And for tensor
	$\AT,\BT$ and matrix $\CB$ we have $\langle \AT,\BT\times_i\CB \rangle = \langle
	\AT_{(i)}\BT_{(i)}^\top,\CB\rangle$. The above equation~(\ref{eqn:horizonspaceEqn2}) is
	equivalent to the following one
	\begin{equation}
	\begin{small}
	\begin{aligned}
	&\sum_{i = 1}^3 \big\langle \UB_i^\top\eta_i\GT_{(i)}\GT_{(i)}^\top
	+ \eta_{\GT}\GT_{(i)}^\top 
	+ N\alpha_i (\UB_i - \PB_i\PB_i^\top\UB_i)^\top\eta_i,\OmegaB_i \big\rangle = 0
	,\forall \text{skew matrix } \OmegaB_i
	\end{aligned}
	\end{small}
	\label{eqn:horizonspaceEqn2}
	\end{equation}
	Thus we have $\eta_{\OXT}$ satisfy the following conditions 
	\begin{equation}
	\begin{small}
	\begin{aligned}
	(\PB_i\PB_i^\top\UB_i)^\top\eta_i\GT_{(i)}\GT_{(i)}^\top 
	+ (\UB_i - \PB_i\PB_i^\top\UB_i)^\top\eta_i (N\alpha_i\IB_i + \GT_{(i)}\GT_{(i)}^\top)
	\text{ is a symmetric matrix} \forall i \in \{1,2,3\}.
	\end{aligned}
	\end{small}
	\end{equation}
	Defining $\VB_i := \PB_i\PB_i^\top\UB_i$, $\WB_i := \UB_i - \PB_i\PB_i^\top\UB_i$,
	$\GB_i := \GT_{(i)}\GT_{(i)}^\top$, $\GB_{\alpha_i} := N\alpha_i\IB_i + \GT_{(i)}\GT_{(i)}^\top$,
	we obtain the formula for the horizontal space:
	\begin{equation}
	\begin{small}
	\begin{aligned}
	\HX =   \big\{\eta_{\OXT} \in T_{\OXT}\MMr | \VB_i^\top\eta_i\GB_i + \WB_i^\top\eta_i\GB_{\alpha_i}
	\text{is symmetric} \big\}
	\end{aligned}
	\end{small}
	\end{equation}
\end{proof}	

\subsubsection{Proof of Prop.~(\ref{prop:diffeom})}
	Suppose $\XT$ has tucker factors $(\GT,\UB_1,\UB_2,\UB_3)$, then one can certify that:
	\[
		\pi^{-1}(\OXT) = [\GT,\UB_1,\UB_2,\UB_3].
	\]
	And hence the equivalent relationship $\sim$ defined by the equivalent classes $[\GT,\UB_1,\UB_2,\UB_3]$
	can also be expressed in terms of the map $\pi(\cdot)$:
	\[
		(\GT,\UB_1,\UB_2,\UB_3) \sim (\HT, \VB_1, \VB_2, \VB_3)\text{ if and only if }
			\pi(\GT,\UB_1,\UB_2,\UB_3) = \pi(\HT, \VB_1, \VB_2, \VB_3).
	\]
	Since $\pi(\cdot)$ is a submersion (see Lemma~\ref{lemma:submersion}),  
	by the submersion theorem (Prop. 3.5.23 of~\cite{abraham2012manifolds}), the equivalent relation $\sim$ defined by the equivalent classes is regular
	and the quotient manifold $\MMr / \sim$ is diffeomorphic to $\FM_r$. And according to the proof of Prop. 3.5.23 of
	~\cite{abraham2012manifolds},
	the mapping $\varrho([\GT,\UB_1,\UB_2,\UB_3]) = \GT \times_{i=1}^3 \UB_i$ defines the diffeomorphism from $\MMr/\sim$
	to $\FM_r$. Therefore, $\rho(\XT) = \varrho^{-1}(\XT) = [\GT, \UB_1, \UB_2, \UB_3]$, where $[\GT, \UB_1, \UB_2, \UB_3]$
	is the tucker representation of $\XT$.
\subsubsection{Proof of Proposition~\ref{prop:equiv_opt} }
Let $\OXT = (\GT,\UB_1,\UB_2, \UB_3)$ be any tucker factors of tensor $\OXT \in \FM_r$. According to the definition of Chordal
distance of subspaces of different dimension~\cite{ye2014distance}, we have
\begin{eqnarray}
	{\dist}^2(\spann(\XT_{(i)}),\spann(\PB_i)) & = & {\dist}^2(\spann(\UB_i),\spann(\PB_i)) \\
	& = & \sum_{i = 1}^{r_i} \sin^2(\theta_i) + k_i - r_i \\
	& = &\sum_{i = 1}^{r_i} (1 - \cos^2(\theta_i)) + k_i - r_i \\
	& = &\tr(\IB) - \| \PB_i^\top\UB_i \|_F^2 + k_i - r_i \\
	& = & \tr(\UB_i^\top(\IB - \PB_i^\top\PB_i)\UB_i) + k_i - r_i
\end{eqnarray}
where in the second equation $\theta_i$ is the $i$-th principal angle between $\spann{\UB_i}$ and $\spann{\PB_i}$,
the second equation is derived from the definition of Chordal distance, 
the fourth equation is derived from the fact that $\cos(\theta_i)$ is the $i$-th singular value
of $\PB_i^\top\QB_i$ due to $\PB_i$ and $\QB_i$ are orthogonal bases (see Alg 12.4.3 of~\cite{golub2012matrix}).
Therefore for all $ (\GT,\UB_1,\UB_2,\UB_3) \in \MMr$, we have
\begin{equation}
	l(\pi(\GT,\UB_1,\UB_2,\UB_3)) = \frac{1}{2} \|\PO( \GT\times_{i=1}^3\UB_i-\RT)\|_F^2  
		+ \sum_{i = 1}^3 ( \tr(\UB_i^\top(\IB - \PB_i^\top\PB_i)\UB_i) + k_i - r_i ).
\end{equation}
Which is equivalent to:
\begin{equation}
l(\pi(\GT,\UB_1,\UB_2,\UB_3)) = f(\GT,\UB_1,\UB_2,\UB_3) + C
	\label{eqn:lossEqCost}
\end{equation}
where $C = \sum_{i=1}^3 (k_i - r_i)$ is a constant.

Note that the critical points of a function $h(x)$ over a smooth manifold $\MM$ are those whose Riemannian
gradient vanishing, that is $\grad h(x) = 0$. And one can show that:
\begin{equation}
	\grad h(x) = 0 \text{ if and only if }Dh(x)[\eta_x] = 0 \forall \eta_x \in T_x\MM.
	\label{condition:gradientEqualZero}
\end{equation}

To prove that $\XT$ is a critical point of $l(\cdot)$ over $\FM_r$ if and only if $[\XT]$ is a critical point
of $\tilde{f}(\cdot)$ over $\MMr / \sim$, we need to prove that
\begin{equation}
	\grad l(\XT) = \zeroB \text{ if and only if } \grad \tilde{f}([\XT]) = \zeroB.
	\label{condition:gradient_vanish}
\end{equation}
Note that since $\grad f(\GT,\UB_1,\UB_2,\UB_3)$ is the horizontal lift of $\grad \tilde{f}([\XT])$
for all $(\GT,\UB_1,\UB_2,\UB_3) \in [\XT]$. We have $\grad \tilde{f}([\XT]) = \zeroB$ if and
only if  $\grad f(\GT,\UB_1,\UB_2,\UB_3) = 0$ for at least one $(\GT,\UB_1,\UB_2,\UB_3) \in [\XT]$.
Thus to prove~(\ref{condition:gradient_vanish}), one only need to certify
\begin{equation}
\grad l(\XT) = \zeroB \text{ if and only if } \exists (\GT,\UB_1,\UB_2,\UB_3) \in [\XT]\text{ such that } \grad f(\GT,\UB_1,\UB_2,\UB_3) = \zeroB.
\label{condition:gradient_vanish_eqv}
\end{equation}

On one side, suppose $\grad l(\XT) = \zeroB$, and $\OXT = (\GT,\UB_1,\UB_2, \UB_3)\in [\XT]$. Let $\eta_\OXT$ be any tangent 
vector belonging to $T_{\OXT} \MMr$. We have:
\begin{eqnarray}
D f(\OXT)[\eta_\OXT] &=& D l(\pi(\OXT))[D\pi(\OXT)[\eta_\OXT]] \\
					 &=& D l(\XT)[D\pi(\OXT)[\eta_\OXT]] \\
					 &=& 0
\end{eqnarray}
where the first equation is derived from equation~(\ref{eqn:lossEqCost}) and chain rule of first order derivative;
the third equation is due to $\grad l(\XT) = \zeroB$ and $D\pi(\OXT)[\eta_\OXT] \in T_{\XT} \FM_r$ since
$\pi(\cdot)$ is a submersion (See Lemma~\ref{lemma:submersion}). 
Because $\eta_\OXT$ is an arbitrary tangent vector, we have
\begin{equation}
	D f(\OXT)[\eta_\OXT] = 0 \forall \eta_\OXT \in T_{\OXT}\MMr.
\end{equation}
And according to~(\ref{condition:gradientEqualZero}) we have $\grad f(\OXT) = \zeroB$. Thus, we prove that
\begin{equation}
  \grad l(\XT) = \zeroB \Rightarrow \grad f(\OXT) = 0.
  \label{condition:eqvHalf_1}
\end{equation}

On the other side, suppose $\OXT = (\GT,\UB_1,\UB_2, \UB_3)\in [\XT]$ and $\grad f(\OXT) = \zeroB$. Then 
for all $\eta_\XT \in T_\XT \FM_r$ we have:
\begin{eqnarray}
	D l(\XT)[\eta_\XT] &=& D l(\pi(\OXT))[\eta_\XT] \\
					   &=& D l(\pi(\OXT))[ D\pi(\OXT)[\eta_\OXT] ] \\
					   &=& D f(\OXT)(\eta_\OXT) \\
					   &=& 0
\end{eqnarray}
where the second equation is because there exist $\eta_\OXT \in T_{\OXT} \MMr$ such that
$D\pi(\OXT)[\eta_\OXT] = \eta_\XT$ due to $\pi()$ being a submersion (See Lemma~\ref{lemma:submersion});
the third equation is derived by the equation~(\ref{eqn:lossEqCost}) and chain rule of first order
derivative; the fourth equation is due to $\grad f(\OXT) = \zeroB$. Thus we have proved that
\begin{equation}
	\grad f(\OXT) = 0 \Rightarrow \grad l(\XT) = 0
	\label{condition:eqvHalf_2}
\end{equation}

Since we have proved both~(\ref{condition:eqvHalf_1}) and (\ref{condition:eqvHalf_2}), we have (\ref{condition:gradient_vanish_eqv}) holds.

\subsubsection{Proof of Proposition~\ref{prop:scale_gradient}}
Since the Euclidean ambient space~(\ref{eqn:ambient_space}) $\RBB^{r_1\times r_2 \times r_3} \times \RBB^{n_1\times r_1} \times \RBB^{n_2\times r_2} \times \RBB^{n_3\times r_3}$ is an special smooth manifold,
with tangent space at each its point being the ambient space itself~\cite{absil2009optimization}. Therefore, one can
endow the ambient space with a metric, and treats it as a Riemannian manifold.
By endowing the ambient space with the same metric with total space, namely:
\begin{equation}
\begin{aligned}
\langle \OXT, \OYT \rangle_{\OZT} 
= \sum_{i = 1}^3 \langle \OXT_{\UB_i}, \OYT_{\UB_i}(\OZT_{\GT})_{(i)}(\OZT_{\GT})_{(i)}^\top \rangle 
+  \langle \OXT_{\GT}, \OYT_{\GT} \rangle 
+\sum_{i = 1}^3 N\alpha_i  \langle \OXT_{\UB_i}, (\IB_i - \PB_i\PB_i^\top)\OYT_{\UB_i} \rangle	
\end{aligned}
\label{eqn:rMetric_ambient}
\end{equation}
where $\OXT,\OYT,\OZT$ are any ambient vectors, and all of them are tuples like $(\OXT_{\GT},\OXT_{\UB_1},\OXT_{\UB_2},\OXT_{\UB_3})$.
The scaled Euclidean of the cost $\tilde{\nabla} f(\OXT)$ means the ambient vector which satisfies the following 
condition
\begin{equation}
\langle \tilde{\nabla}f(\OXT),\OYT\rangle_\OXT = Df(\OXT)[\OYT], \forall \OYT \in \text{ambient space}
\label{eqn:gradient:1}
\end{equation}
This equation is equivalent to the following:
\begin{equation}
\begin{aligned}
&\sum_{i = 1}^3 \langle \OYT_i , \tilde{\nabla}_{\UB_i}f(\OXT)(\OXT_{\GT})_{(i)}(\OXT_{\GT})_{(i)}^\top \rangle +  \langle \OYT_\GT, \tilde{\nabla}_{\GT}f(\OXT) \rangle 
+ \sum_{i = 1}^3 N\alpha_i  \langle \OYT_i, (\IB_i - \PB_i\PB_i^\top)\tilde{\nabla}_{\UB_i}f(\OXT) \rangle	\\
& = \sum_{i = 1}^3 \langle \OYT_i, \nabla_{\UB_i}\barf(\OXT) \rangle
+ \langle\OYT_{\GT}, \nabla_{\GT}\barf(\OXT)\rangle, \forall \OYT \in \text{ ambient space}
\end{aligned}
\label{eqn:gradient:2}.
\end{equation}
By taking the partial Euclidean gradient both side of above equation with respect to $\OYT_{\GT}$ and $\OYT_{\UB_i}$, one has
\begin{equation}
\begin{aligned}
& \tilde{\nabla}_{\GT}\barf(\OXT) = \nabla_{\GT}\barf(\OXT) \\
& \tilde{\nabla}_{\UB_i} \barf(\OXT) = \EB_i ((\OXT_{\GT})_{(i)}(\OXT_{\GT})_{(i)}^\top)^{-1} 
+ \FB_i (N\alpha_i\IB_i+(\OXT_{\GT})_{(i)}(\OXT_{\GT})_{(i)}^\top)^{-1}
\end{aligned}
\end{equation}
where $\EB_i = \PB_i\PB_i^\top\nabla_{\UB_i}f(\OXT)$ and $\FB_i = \nabla_{\UB_i}f(\OXT) - \EB_i$.

\subsubsection{Proof of Proposition~\ref{prop:metric} }
According to~\clr{\cite{absil2009optimization}}, to prove $\MMr / \sim$ has the structure of Riemannian
manifolds, one need to show that for all $[\XT] \in \MMr / \sim$ and for all tangent vectors
 $\eta_{[\XT]}, \xi_{[\XT]}\in T_{[\XT]} \MMr / \sim$ we have
 \begin{equation}
	 \langle \eta_{\OXT_1}, \xi_{\OXT_1} \rangle_{\OXT_1} = \langle \eta_{\OXT_2}, \xi_{\OXT_2} \rangle_{\OXT_2},
	 \forall \OXT_1, \OXT_2 \in [\XT] 
	 \label{eqn:rMetric_invariance}
 \end{equation}
 where $\eta_{\OXT_1},\eta_{\OXT_2}$ are horizontal lift of $\eta_{[\XT]}$ and
 $\xi_{\OXT_1}, \xi_{\OXT_2}$ are horizontal lift of $\xi_{[\XT]}$.
To prove that, we firstly express $\OXT_2, \eta_{\OXT_2}, \xi_{\OXT_2}$ in terms of
$\OXT_1, \eta_{\OXT_1}, \xi_{\OXT_1}$, then verify the invariant property~(\ref{eqn:rMetric_invariance}). 

Let $\OXT_1 = (\GT,\UB_1,\UB_2,\UB_3)$. since $\OXT_1,\OXT_2 \in[\XT]$, there exist orthogonal matrices $\OB_i \in \OM(r_i)$
such that
\begin{equation}
 \OXT_2 = (\GT\times_{i=1}^3 \OB_i^\top, \UB_1\OB_1,\UB_2\OB_2,\UB_3\OB_3).
 \label{eqn:OYTandOXT}
\end{equation}

Let $\eta_{\OXT_1} = (\eta_{\GT},\eta_1,\eta_2,\eta_3)$, in this paragraph, we will prove that $\eta_{\OXT_2}$ can be
expressed by the following formula
\begin{equation}
	 \eta_{\OXT_2} = (\eta_{\GT} \times_{i=1}^3 \OB_i^\top, \eta_1\OB_1,\eta_2\OB_2,\eta_3\OB_3)
	 \label{eqn:oetaOYEqualto}.
\end{equation} 
Note that $\eta_{\OXT_2}$ is the horizontal lift of $\eta_{[\OXT]}$, to prove (\ref{eqn:oetaOYEqualto}), 
one only need to show that $(\eta_{\GT} \times_{i=1}^3 \OB_i^\top, \eta_1\OB_1,\eta_2\OB_2,\eta_3\OB_3)$ 
satisfy the following two conditions (See Sec. 3.6.2 of~\cite{absil2009optimization})
\begin{eqnarray}
	\zeta &\in& \HT_{\OXT_2} \\
	D\tau(\OXT_2)[\zeta] &=& \eta_{[\XT]} \label{eqn:zetaIsHorizonLift}
\end{eqnarray}
where for brevity we denote $(\eta_{\GT} \times_{i=1}^3 \OB_i^\top, \eta_1\OB_1,\eta_2\OB_2,\eta_3\OB_3)$ by $\zeta$;
the $\HT_{\OXT}$ is the horizontal space at $\OXT$ (See Lemma~\ref{Prop:horizontal_space} for its expression);
$\tau(\cdot)$ is the nature mapping from $\MMr$ to $\MMr/\sim$ which is defined by
\[
	\tau(\OXT) = [\XT]
\]
Note that $\tau()$ is a composition of map $\rho(\cdot)$ and map $\pi(\cdot)$ defined in Prop.~\ref{prop:diffeom} and Lemma~\ref{lemma:submersion}, namely 
\begin{equation}
	\tau(\OXT) = \rho(\pi(\OXT)).
\end{equation}
According to Lemma.~\ref{Prop:horizontal_space}, $\HT_{\OXT_1} = \{\eta_{\OXT_1} \in T_{\OXT_1} \MMr | \VB_i^\top\eta_i\GB_i + \WB_i^\top\eta_i\GB_{\alpha_i}\text{ is symmetric}\}$ where $\VB_i = \PB_i\PB_i^\top\UB_i, \WB_i = \UB_i - \VB_i, \GB_i = \GT_{(i)}\GT_{(i)}^\top$ and $\GT_{\alpha_i} = N\alpha_i\IB_i + \GT_{(i)}\GT_{(i)}^\top$.
Using the equation~(\ref{eqn:OYTandOXT}), we have:
\begin{equation}
	\HT_{\OXT_2} = \{\eta_{\OXT_2} \in T_{\OXT_2} \MMr | \OB_i^\top\VB_i^\top\eta_i \OB_i^\top\GB_i\OB_i + \OB_i^\top\WB_i^\top\eta_i\OB_i^\top\GB_{\alpha_i}\OB_i\text{ is symmetric}\}
\end{equation}
(Note that when proving the above equation, we use the equations like: 
	$(\GT\times_{i=1}^3 \OB_i^\top)_{(1)} = \OB_1^\top \GT_{(1)} ( \OB_{2}^\top\otimes \OB_{3}^\top )^\top $~\cite{kolda2009tensor}
and the properties like $ \OB_{2}^\top\otimes \OB_{3}^\top$ is orthogonal matrix ).
To prove $\zeta \in \HT_{\OXT_2}$, on one hand we noticed that:
\begin{eqnarray}
	\zeta_i^\top \UB_i\OB_i + \OB_i^\top \UB_i^\top \zeta_i & = & \OB_i^\top\eta_i^\top \UB_i\OB_i + \OB_i^\top \UB_i^\top \eta_i\OB_i \\
															& = & \OB_i^\top(\eta_i^\top \UB_i + \UB_i^\top \eta_i)\OB_i \\
															& = & \zeroB
\end{eqnarray}
where the first equation use the fact $\zeta_i = \eta_i \OB_i$, the third equation use the fact $\eta_i \in T_{\UB_i} \st(r_i,n_i)$
is equivalent to $\eta_i^\top \UB_i + \UB_i^\top \eta_i = \zeroB$ (See Sec 3.5.7 of ~\cite{absil2009optimization}). The above equation
implies that $\zeta_i \in T_{\UB_i\OB_i}\st(r_i,n_i)$. And hence we have
\begin{equation}
\zeta \in \big(\RBB^{r_1\times r_2 \times r_3} \times T_{\UB_1\OB_1} \st(r_1,n_1) \times T_{\UB_2\OB_2} \st(r_2,n_2)  
\times T_{\UB_3\OB_3} \st(r_3,n_3)\big) = T_{\OXT_2} \MMr.
\end{equation}
One the other hand, we have $\OB_i^\top\VB_i^\top\zeta_i \OB_i^\top\GB_i\OB_i + \OB_i^\top\WB_i^\top\zeta_i\OB_i^\top\GB_{\alpha_i}\OB_i$
is symmetric since:
\begin{eqnarray*}
	(\OB_i^\top\VB_i^\top\zeta_i \OB_i^\top\GB_i\OB_i + \OB_i^\top\WB_i^\top\zeta_i\OB_i^\top\GB_{\alpha_i}\OB_i)^\top
		&=& (\OB_i^\top\VB_i^\top\eta_i\OB_i \OB_i^\top\GB_i\OB_i + \OB_i^\top\WB_i^\top\eta_i\OB_i\OB_i^\top\GB_{\alpha_i}\OB_i)^\top\\
		&=& \OB_i^\top(\VB_i^\top\eta_i\GB_i +\WB_i^\top\eta_i\GB_{\alpha_i})^\top\OB_i \\
		&=& \OB_i^\top(\VB_i^\top\eta_i\GB_i +\WB_i^\top\eta_i\GB_{\alpha_i})\OB_i \\
		&=& \OB_i^\top(\VB_i^\top\eta_i\OB\OB^\top\GB_i +\WB_i^\top\eta_i\OB_i\OB_i^\top\GB_{\alpha_i})\OB_i \\
		&=& \OB_i^\top\VB_i^\top\zeta_i \OB_i^\top\GB_i\OB_i + \OB_i^\top\WB_i^\top\zeta_i\OB_i^\top\GB_{\alpha_i}\OB_i.
\end{eqnarray*}
Thus, we have proved that $\zeta \in \HM_{\OXT_2}$. The following equations verify~(\ref{eqn:zetaIsHorizonLift}) holds.
\begin{eqnarray}
	D\tau(\OXT_2)[\zeta] & = & D\rho(\pi(\OXT_2))[D\pi(\OXT_2)[\zeta]] \\
						 & = & D\rho(\XT)\left[ \zeta_\GT \times_{i=1}^3 \UB_i\OB_i + \sum_{i=1}^3 (\GT \times_{i=1}^3\OB_i^\top)\times_i \zeta_i\times_{1\leq j\leq 3, j \neq i} \UB_i\OB_i \right]\\
						 & = & D\rho(\XT)\left[ (\eta_G\times_{i=1}^3\OB_i^\top) \times_{i=1}^3 \UB_i\OB_i + \sum_{i=1}^3 (\GT \times_{i=1}^3\OB_i^\top)\times_i \eta_i\OB_i\times_{1\leq j\leq 3, j \neq i} \UB_i\OB_i \right] \nonumber \\
						 & = & D\rho(\XT)\left[\eta_G \times_{i=1}^3 \UB_i + \sum_{i=1}^3 (\GT \times_{i=1}^3\OB_i)\times_i \eta_i\times_{1\leq j\leq 3, j \neq i} \UB_i \right] \\
						 & = &  D\rho(\XT)[D\pi(\OXT_1)[\eta_{\OXT_1}]] \\
						 & = & D\rho(\pi(\OXT_1))[D\pi(\OXT_1)[\eta_{\OXT_1}]] \\
						 & = & D\rho(\OXT_1)[\eta_{\OXT_1}] \\
						 & = & \eta_{[\XT]}
\end{eqnarray}
where the first equation is derived by the chain rule of derivative, the second equation is derived by using
(\ref{eqn:derivative_of_pi_1}),
the third equation is obtained by using our definition of $\zeta$, the fourth equation is using the property of
tensor matrix product that $\AT\times_i \AB \times_i \BB = \AT \times_i (\BB\AB)$ and 
$\AT\times_i \AB \times_j \BB = \AT \times_j \BB \times_i \AB \forall j\neq i$~\cite{kolda2009tensor},
the fifth equation result from~(\ref{eqn:derivative_of_pi_1}), the eighth equation is because 
$\eta_{\OXT_1}$ is the horizontal lift of $\eta_{[\XT]}$.

By similar arguments of above paragraph, one can verify that
\begin{eqnarray}
	 \xi_{\OXT_2} = (\xi_{\GT} \times_{i=1}^3 \OB_i^\top, \xi_1\OB_1,\xi_2\OB_2,\xi_3\OB_3)
	 \label{eqn:oxiOYEqualto}.	
\end{eqnarray}

Now we have
\begin{eqnarray}
    \langle \eta_{\OXT_2}, \xi_{\OXT_2} \rangle_{\OXT_2} 
    &=& \sum_{i = 1}^3 \langle \eta_i\OB_i, \xi_{i}\OB_i(\GT\times_{i=1}^3\OB_i^\top)_{(i)}(\GT\times_{i=1}^3\OB_i^\top)_{(i)}^\top \rangle + 
     \langle \etaG\times_{i=1}^3\OB_i, \xiG\times_{i=1}^3\OB_i \rangle \nonumber\\ 
    && \quad + \sum_{i = 1}^3 N\alpha_i  \langle \eta_i\OB_i, (\IB_i - \PB_i\PB_i^\top)\xi_i\OB_i
    \rangle \\	
    &=& \sum_{i = 1}^3 \langle \eta_i\OB_i, \xi_{i}\OB_i \OB_i^\top \GB_{(i)}\GB_{(i)}^\top\OB_i \rangle + 
    \langle \etaG\times_{i=1}^3\OB_i, \xiG\times_{i=1}^3\OB_i \rangle \nonumber\\ 
    && + \sum_{i = 1}^3 N\alpha_i  \langle \eta_i\OB_i, (\IB_i - \PB_i\PB_i^\top)\xi_i\OB_i
    \rangle \nonumber \\	
    &=& \sum_{i = 1}^3 \langle \eta_i, \xi_{i} \GB_{(i)}\GB_{(i)}^\top \rangle +  \langle \etaG, \xiG\rangle 
    + \sum_{i = 1}^3 N\alpha_i  \langle \eta_i, (\IB_i - \PB_i\PB_i^\top)\xi_i
    \rangle \\	
    & = & \langle \eta_{\OXT_1}, \xi_{\OXT_1}\rangle_{\OXT_1}
\end{eqnarray}
where the first equation use the expressions of $\OXT_2, \eta_{\OXT_2}, xi_{\OXT_2}$ in terms of
$\OXT_1, \eta_{\OXT_1}, \xi_{\OXT_1}$ (see equations (\ref{eqn:OYTandOXT},\ref{eqn:oetaOYEqualto},\ref{eqn:oxiOYEqualto}));
the second equation is derived by using equations like
\[
	\begin{aligned}
		(\GT\times_{i=1}^3\OB_i^\top)_{(1)}(\GT\times_{i=1}^3\OB_i^\top)_{(1)}^\top = 
			\left(\OB_1^\top \GT_{(1)} (\OB_3^\top \otimes \OB_2^\top)^\top\right)
			\left(\OB_1^\top \GT_{(1)} (\OB_3^\top \otimes \OB_2^\top)^\top\right)^\top 
			 =  \OB_1^\top  \GT_{(1)}  \GT_{(1)}^\top \OB_1;
	\end{aligned}
\]
the third equation is derived from the fact that Euclidean inner product is orthogonal invariant.
And the invariant property of the proposed metric is being proved.
\subsection{ Derivation of The Expressions of Optimization Related Objects }

\subsubsection{Projector from ambient space onto tangent space}\label{sec:appendix:tangent_space}
We call the Euclidean space 
\begin{equation}
	\RBB^{r_1\times r_2 \times r_3} \times \RBB^{n_1\times r_1} \times \RBB^{n_2\times r_2} \times \RBB^{n_3\times r_3}
	\label{eqn:ambient_space}
\end{equation}
the \emph{ambient space}. 
The vector belonging to ambient space is called by \emph{ambient vector}.
One ambient vector is denoted by $(\ZB_\GT,\ZB_1,\ZB_2,\ZB_3)$, for brevity
the notation may be shorted to $\ZB$.
\begin{proposition}
	Let $\MMr$ be the total space, endowed with the Riemannian metric~(\ref{eqn:rMetric}).
	Let $\OXT = (\GT,\UB_1,\UB_2,\UB_3) \in MMr$
	Then the orthogonal projection of an ambient vector $(\ZB_\GT,\ZB_{1},\ZB_{2},\ZB_{3})$ onto the tangent
	space $T_{\OXT}\MMr$ can be computed by
	\begin{equation}
	\begin{aligned}
	\Psi_\OXT(\ZB_\GT,\ZB_{1},\ZB_{2},\ZB_{3}) 
	& = &\big(\ZB_\GT,
	\ZB_{1} - \VB_1\SB_1(\GT_{(1)}\GT_{(1)}^T)^{-1} - \WB_1\SB_1(\GT_{(1)}\GT_{(1)}^T + \alpha_1N \IB_1)^{-1} \\
	& &  \ZB_{2} - \VB_2\SB_2(\GT_{(2)}\GT_{(2)}^T)^{-1} - \WB_2\SB_2(\GT_{(2)}\GT_{(2)}^T + \alpha_2N \IB_2)^{-1} \\
	& &  \ZB_{3} - \VB_3\SB_3(\GT_{(3)}\GT_{(3)}^T)^{-1} - \WB_3\SB_3(\GT_{(3)}\GT_{(3)}^T + \alpha_3N \IB_3)^{-1}
	\big)
	\end{aligned}
	\end{equation}
	where $\VB_i = \PB_i\PB_i^\top \UB_i$ and $\WB_i = \UB_i - \VB_i$ and $\SB_i$ is the solution of the
	following
	matrix linear equation
	\begin{equation}
		\begin{cases}
		\begin{aligned}
		&	\symm(\VB_i^T\VB_i \SB_i (\GTi\GTi^\top)^{-1}) - \symm(\UB_i^\top\ZB_{i}) 
		+ \symm(\WB_i^T\WB_i\SB_i
		(N\alpha_i\IB_i+\GTi\GTi^\top)^{-1}) = \zeroB.
		\end{aligned}\\
		\SB_i = \SB_i^\top
		\end{cases}
		\label{eqn:linearEquationS}
	\end{equation}
	in which $\symm(\AB) = 1/2 (\AB + \AB^\top)$ for all square matrices.	
\end{proposition}
\begin{proof}
	The orthogonal projection of an ambient vector to the tangent space, is computed by
	subtraction of its component belongs to the normal space.
	To begin with we derive the normal space $N_{\OXT}$ which orthogonal complement of $T_{\OXT}\MMr$ with respect
	to the Riemannian metric~(\ref{eqn:rMetric}). Let $\zeta = (\zeta_\GB,\zeta_1,\zeta_2,\zeta_3)
	\in N_{\OXT}$ be any vector of the normal space. Then we have
	\begin{equation}
	\langle \zeta,\eta \rangle_{\OXT} = 0 \forall \eta \in T_{\OXT}\MMr
	\label{eqn:projector:1}
	\end{equation}
	Since the tangent space of total space can be expressed as
	\begin{equation}
	\begin{aligned}
		T_{\OXT}\MMr = \RBB^{r_1\times r_2 \times r_3} \times T_{\UB_1}\st(r_1,n_1) \times T_{\UB_2}\st(r_2,n_2) \times T_{\UB_3}\st(r_3,n_3)
	\end{aligned}
	\label{eqn:projector:2}
	\end{equation}
	where the tangent space of Stiefel manifold can be formulated as:
	\begin{equation}
	\begin{aligned}
	 T_{\UB_i}\st(r_i,n_i) = \{\UB_i\OmegaB_i + \UB_{i,\perp} \KB_i | 
	    \OmegaB_i \in \RBB^{r_i \times r_i}\text{ is skew and } \KB_i \in \RBB^{(n_i-r_i)\times r_i}\}
	\end{aligned}
	\label{eqn:projector:3}
	\end{equation}
	and $\UB_{i,\perp}$ is also a matrix with orthogonal columns such that
	$\UB_{i,\perp}^\top \UB_i = 0$. Using the formula~(\ref{eqn:projector:2})
	and (\ref{eqn:projector:3}), the equation~(\ref{eqn:projector:1}) is equivalent
	to the following formula
	\begin{equation}
	\begin{aligned}
	&\sum_{i = 1}^3 \langle \UB_i\OmegaB_i + \UB_{i,\perp} \KB_i , \zeta_i\GT_{(i)}\GT_{(i)}^\top \rangle +  \langle \etaG, \zeta_\GT \rangle + 
	\sum_{i = 1}^3 N\alpha_i  \langle \UB_i\OmegaB_i + \UB_{i,\perp} \KB_i , (\IB_i - \PB_i\PB_i^\top)\zeta_i \rangle = 0 \\	
	&\quad \quad \forall \KB_i \in \RBB^{r_i \times (n_i - r_i)}, \text{skew matrix }\OmegaB_i\in \RBB^{r_i\times r_i}, \eta_{\GT} \in \RBB^{r_1\times r_2 \times r_3}.
	\end{aligned}
	\label{eqn:projector:4}
	\end{equation}
	Using the fact that the condition $\langle \ZB, \UB_i\OmegaB_i + \UB_{i,\perp}\KB_i \rangle = 0
	\forall \KB_i\text{ and skew matrix} \OmegaB_i$ is equivalent to that $\ZB = \UB_i \SB_i$
	where $\SB_i$ is any symmetric matrix.
	the above equation~(\ref{eqn:projector:4}) can be simplified as the following conditions
	\begin{equation}
	\begin{cases}
	 \zeta_\GT = 0 \\
	 \PB_i\PB_i^\top\zeta_i\GTi\GTi^\top
	+ (\IB_i - \PB_i\PB_i^\top)\zeta_i
	(N\alpha_i\IB_i + \GTi\GTi^\top) = \UB_i\SB_i
	 \forall i \in \{1,2,3\}  \label{eqn:projector:5}
	\end{cases}
	\end{equation}
	where $\SB_i$ is a symmetric matrix.  Note that  the second equation of~(\ref{eqn:projector:5}),
	is equivalent to the following equations:
	\begin{equation}
		\begin{aligned}
			\PB_i\PB_i^\top\zeta_i &=& \VB_i\SB_i(\GTi\GTi^\top)^{-1} \\
			(\IB_i - \PB_i\PB_i^\top)\zeta_i& = & \WB_i\SB_i(N\alpha_i\IB_i + \GTi\GTi^\top)^{-1}
		\end{aligned}
		 \label{eqn:projector:6}
	\end{equation}
	where $\VB_i = \PB_i\PB_i^\top\UB_i$ and $\WB_i = \UB_i - \VB_i$, and the first equation
	is obtained by multiplying the both side of the second formula of~(\ref{eqn:projector:5})
	by $\PB_i\PB_i^\top$, the second equation is obtained by multiplying the both side of
	the second formula of~(\ref{eqn:projector:5}) by $\IB - \PB_i\PB_i^\top$.	
	The above equation array is further equivalent to
	\begin{equation}
	 \zeta_i =  \VB_i\SB_i(\GTi\GTi^\top)^{-1}
	+ \WB_i\SB_i(N\alpha_i\IB_i + \GTi\GTi^\top)^{-1}
	\label{eqn:projector:7}
	\end{equation}
	since one can obtain equation (\ref{eqn:projector:7}) by adding the two equations in (\ref{eqn:projector:6}),
	and one an obtain the two equations in (\ref{eqn:projector:6}) via multiplying both sides of 
	(\ref{eqn:projector:7}) by $\PB_i\PB_i^\top$ or $\IB -\PB_i\PB_i$.
	Therefore, the normal space $N_{\OXT}$ can be expressed as follows.
	\begin{equation}
	\begin{aligned}
	&N_{\OXT} = \big\{(0,\zeta_1,\zeta_2,\zeta_3) | \zeta_i =  \VB_i\SB_i(\GTi\GTi^\top)^{-1} 
	 +\WB_i\SB_i (N\alpha_i\IB_i + \GTi\GTi^\top)^{-1}
	,\SB_i = \SB_i^\top,1\leq i\leq3\big\}.
	\end{aligned}
	\end{equation}
	Now the projection of an ambient vector can be calculated by subtracting its
	components in the normal space $N_{\OXT}$. Specifically,
	suppose $\Psi_{\OXT}(\ZB_\GT,\ZB_1,\ZB_2,\ZB_3) = (\YB_\GT,\YB_1,\YB_2,\YB_3)$,
	we have $\YB_\GT = \ZB_\GT$ and there exist symmetric matrices $\SB_i$ such that
	\begin{equation}
	\YB_i = \ZB_i - \VB_i\SB_i(\GTi\GTi^\top)^{-1} - \WB_i\SB_i (N\alpha_i\IB_i + \GTi\GTi^\top)^{-1}
	\label{eqn:projector:8}
	\end{equation}
	where $\VB_i = \PB_i\PB_i^\top\UB_i$, $\WB_i = \UB_i - \VB_i$ and $1\leq i \leq 3$.
	Since
	$(\YB_\GT,\YB_1,\YB_2,\YB_3)\in \TOMM$ we have
	\begin{equation}
	\UB_i^\top\YB_i + \YB_i^\top\UB_i = 0, 1\leq i \leq 3.
	\end{equation}
	By plugging in the equation~(\ref{eqn:projector:8}) into the above equation
	we can obtain the linear equations for the symmetric matrix $\SB_i$:
	\begin{equation}
	\begin{aligned}
	\symm(\VB_i^T\VB_i \SB_i (\GTi\GTi^\top)^{-1}) - \symm(\UB_i^\top\ZB_{i}) 
	+ \symm(\WB_i^T\WB_i\SB_i
	(N\alpha_i\IB_i+\GTi\GTi^\top)^{-1}) = \zeroB.
	\end{aligned}
	\end{equation}
	
\end{proof}

\subsubsection{Projector from Tangent Space onto Horizontal Space}\label{sec:appendix:projector_hor}
\begin{proposition}
	Let $\MMr$ be the total space, endowed with the Riemannian metric~(\ref{eqn:rMetric}).
	Let $\OXT = (\GT,\UB_1,\UB_2,\UB_3) \in MMr$.	
	Then the orthogonal projector $\Pi_{\OXT}$ from tangent space $T_{\OXT} \MMr$ to horizontal space $\HX$ 
	has the following form
	\begin{equation}
	\begin{aligned}
	\Pi_{\OXT}(\etaOX) = (\etaG + \sum_{i = 1}^{3} \GT \times_i \OmegaB_i, \etaU - \UB_1\OmegaB_1,
	\etaV - \UB_2\OmegaB_2, 
	\etaW - \UB_3\OmegaB_3 )
	\end{aligned}
	\end{equation}
	where $\eta_\OXT = (\etaG,\etaU,\etaV,\etaW)$ is a tangent vector. And $(\OmegaB_1,\OmegaB_2,\OmegaB_3)$ is the solution of
	the following linear matrix equation system:
	\begin{equation}
	\begin{cases}
	\begin{aligned}
	&\skewm(\VB_i^T\VB_i\OmegaB_i\GTi\GTi^\top) + \skewm(\GTi\GTi^\top\OmegaB_i) 
	+ \skewm(\WB_i^\top\WB_i\OmegaB_i(N\alpha_i\IB_i + \GTi\GTi^\top)) \\
	& \quad\quad - \GTi(\IB_{j_i}\otimes\OmegaB_{k_i} + \OmegaB_{j_i} \otimes \IB_{k_i})\GTi^\top \\
	& \quad \quad = \skewm[\VB^\top\etaUi\GTi\GTi^\top + \WB^\top\etaUi(N\alpha_i\IB_i+\GTi\GTi^\top) 
	+\GTi(\etaG)_{(i)}^\top] \forall i \in \{1,2,3\} \\
	& \OmegaB_i^\top = - \OmegaB_i \forall i \in \{1,2,3\}
	\end{aligned}
	\end{cases}
	\label{eqn:linear_system_omega}
	\end{equation}
	where $j_i  = \max\{k|k\in\{1,2,3\}, k\neq i\}$ and $k_i = \min\{k|k\in\{1,2,3\}, k\neq i\}$, 
	$\VB_i = \PB_i\PB_i^\top \UB_i$ and $\WB_i = \UB_i - \VB_i$.
\end{proposition}
\begin{proof}
	The projection from tangent space $T_{\OXT}\MMr$ onto the horizontal space $\HX$ is
	also derived by subtracting the normal component from the tangent vector.
	Note that the normal space to $\HX$ in $T_{\OXT}\MMr$ is the vertical space $\VX$ 
	defined in~(\ref{eqn:vertical_space}).
	Then the projection $\Psi(\etaOX) = (\varsigma_{\GT}, \varsigma_1,\varsigma_2,\varsigma_3)$
	have the following form:
	\begin{eqnarray}
	\begin{cases}
	\varsigma_{\GT} = \eta_\GT + \sum_{i=1}^3\GT\times_i\OmegaB_i, \\
	\varsigma_{i} = \eta_i - \UB_i\OmegaB_i \forall i = \{1,2,3\}
	\end{cases}
	\end{eqnarray}
	where $\OmegaB_i$ is a skew matrix to be determined. Since $(\varsigma_{\GT}, \varsigma_1,\varsigma_2,\varsigma_3)\in \HX$, 
	then according to Prop.~\ref{Prop:horizontal_space}, it must satisfy that:
	\begin{equation}
		\skewm(\VB_i^\top\varsigma_i\GB_i + \WB_i^\top\varsigma_i\GB_{\alpha_i}) = \zeroB  \forall 1\leq i \leq 3 
	\end{equation}
	where $\VB_i = \PB_i\PB_i^\top\UB_i$, $\WB_i = \UB_i - \PB_i\PB_i^\top\UB_i$,
	$\GB_i = \GT_{(i)}\GT_{(i)}^\top$, $\GB_{\alpha_i} = N\alpha_i\IB_i + \GT_{(i)}\GT_{(i)}^\top$,
	and $\skewm()$ is a map define on square matrices, $\skewm(\AB) = 1/2 (\AB - \AB^\top)$.
	Doing some algebra, we obtain the linear system~(\ref{eqn:linear_system_omega}).
	
\end{proof}

\subsubsection{Retraction}\label{sec:appendix:retraction}
We proof that the retraction is compatible with the metric~(\ref{eqn:rMetric}) by showing it
induce an retraction over the quotient manifold.
\begin{lemma}
	Let $R_{\cdot}(\cdot)$ be the retraction defined in~(\ref{eqn:retraction}).
	Then
	\[
		E_{[\XT]}(\eta_{[\XT]}) := [R_{\OXT}(\eta_{\OXT})]
	\]
	where $\OXT \in [\XT]$ and $\eta_{\OXT}$ is a horizontal lift of $\eta_{[\XT]}$,
	defines an retraction over the quotient manifold $\MMr / \sim$.
\end{lemma}
\begin{proof}
	Let $\OXT_1,\OXT_2$ be any tucker factors belonging to equivalent classes $[\XT]$.
	Let $\eta_{[\XT]}$ be any tangent vector in the tangent space $T_{[\XT]} \MMr /\sim$.
	Let $\eta_{\OXT_1}$ and $\eta_{\OXT_2}$ are horizontal lifts of $\eta_{[\XT]}$.
	Suppose $\OXT_1 = (\GT,\UB_1,\UB_2,\UB_3)$, then we have
%	(\GT + \etaG, \{\uf(\UB_i+\eta_{i})\}_{i=1}^3 )
	\begin{eqnarray}
		[R_{\OXT_2}(\eta_{\OXT_2})] &=& \left[R_{(\GT\times_{i=1}^3\OB_i^\top,\{\UB_i\OB_i\})}( \eta_{\GT} \times_{i=1}^3 \OB_i^\top, \{\eta_i\OB_i\}_{i=1}^3 )\right] \\
		&=& \left[ (\GT+\eta_{\GT})\times_{i=1}^3\OB_i^\top, \{ \uf(\UB_i\OB_i + \eta_i\OB_i) \}_{i=1}^3 \right] \\
		&=& \left[ (\GT+\eta_{\GT})\times_{i=1}^3\OB_i^\top, \{  \uf(\UB_i+ \eta_i)\OB_i \}_{i=1}^3 \right] \\
		&=& \left[ (\GT+\eta_{\GT}), \{  \uf(\UB_i+ \eta_i) \}_{i=1}^3 \right] \\
		&=& [R_{\OXT_1}(\eta_{\OXT_1})]
	\end{eqnarray}	
	where the first equation is because of~(\ref{eqn:OYTandOXT}) and (\ref{eqn:oetaOYEqualto}),
	the second equation use the definition of retraction~(\ref{eqn:retraction}),
	the third equation is because $\uf(\AB\OB) = \uf(\AB)\OB$ for all orthogonal matrix $\OB$.
	Thus, according to Prop 4.1.3 of~\cite{absil2009optimization}, we have that $E_{\cdot}(\cdot)$ 
	is a valid retraction of $\MMr /\sim$.eqn:retraction
	
\end{proof}

\subsubsection{ The Euclidean Gradient of the Cost}
The Euclidean gradient of the cost $\nabla f(\GT,\UB_1,\UB_2,\UB_3)$ can be 
decompose as $\nabla f(\GT,\UB_1,\UB_2,\UB_3) = (\nabla_\GT f, \nabla_{\UB_1} f, \nabla_{\UB_2} f, \nabla_{\UB_3} f)$
where $\nabla_\GT f$ and $\nabla_{\UB_i} f$ are partial derivatives of the cost with respect to $\GT$ and $\UB_i$.
By doing some algebra, one has:
\begin{equation}
\begin{aligned}
&\nabla_\GT f(\GT,\UB_1,\UB_2,\UB_3) = \ST \times_{i = 1}^3 \UB_i^\top  \\
&\nabla_{\UB_i} f(\GT,\UB_1,\UB_2,\UB_3) = \ST_{(i)} (\UB_{j_i}\otimes \UB_{k_i}) \GTi + N\alpha_i \WB_i
\end{aligned}
\label{eqn:gradient:3}
\end{equation}
where
\begin{equation}
\begin{aligned}
&\ST = \PO( \GT \times_{i=1}^3 \UB_i -\RT) \\
&\WB_i = \UB_i - \PB_i\PB_i^\top\UB_i,
\end{aligned}
\end{equation}
$j_i  = \max\{k|k\in\{1,2,3\}, k\neq i\}$ and $k_i = \min\{k|k\in\{1,2,3\}, k\neq i\}$.

\subsection{More Empirical Results: Simulation}
In the simulations, we complete a random tensor $\RT$ whose size is fixed to $5000\times 5000 \times 5000$ and
multilinear rank to $(10,10,10)$. And it is generated
by $\RT = \AT \times_1 \BB_1 \times_2 \BB_2 \times_3 \BB_3$ where $\AT\in\RBB^{10\times10\times10}$
and $\BB_i \in \RBB^{5000\times 10}$ are random (multi-dimensional) arrays with i.i.d standard Gaussian entries. The
side informations are encoded in three feature matrices. They are generated
by $\FB_i = \BB_i + s\|\BB_i\|_F\NB_i$ where $\NB_i$ is a noise matrix with entries drew from i.i.d normal distribution. The indices of the observed entries $\Omega$ are sampled from the full indices set of
the $5000\times5000\times5000$ tensor uniformly at random. Its cardinality $|\Omega|$ is set to $OS\times D$ where $D= 3\times(5000\times10 - 10^2) + 10^3$ is the dimension of the manifolds of $5000\times 5000 \times 5000$ tensors with multilinear rank $(10,10,10)$ and $OS$ is called the Over-Sampling ratio. We compare the five tensor completion solvers under the following four scenarios. In each run
the compared solvers are started with the same initializer generated from random,
and stopped when either the norm of the gradient is less than $10^{-4}$ or the number of iterations is more than 300.
To show the effectiveness of the propose metric, we also implemented an Riemannian CG solver, with the least square metric~\cite{kasai2016low}.
 And the parameters of $CGSI$ and $FTCSI$ are set to the same values as they solve the same problem.
\subsubsection{Case 1: influence of sampling ratio}
We study the number of observed samples on the performance of the compared solvers.
We vary the oversampling ratio in the set $OS \in \{0.1,1,5\}$ while fixing the noise scale of the feature matrices to $10^{-5}$. Then, run the five solvers on each tasks. For each run, we set $\alpha_i, 1\leq i \leq 3$
are all set to $10/|\Omega|$ and $\lambda = 0$ for CGSI and FTCSI.
The parameters of other baselines are set to the defaults.
We report the convergence behavior of the compared solvers in Fig.~\ref{fig:simulation}(a-c). Note that in Fig.~\ref{fig:simulation}(a) the RMSE curve of FTC coincides with that of GE and
in Fig.~\ref{fig:simulation}(c) the RMSE curve of FTC coincides with that of FTCSI. From Fig.~\ref{fig:simulation}(a) and (b), we can see that
only CGSI and FTCSI successfully bring the RMSE down below $10^{-2}$ when  OS is smaller than $1$. This shows that when the observed entries are scarce, using the side information in the optimization can make a big difference on the accuracy of tensor completion task.
And from Fig.~\ref{fig:simulation}(a-c), we can see that CGSI converges to the solution faster than FTCSI. This is shows that our proposed metric can indeed accelerate the convergence of Riemannian conjugate gradient
descent method.

\subsubsection{Case 2: influence of noisy side information} To study the affect of noisy feature matrix
on the performance of the proposed method. We fix the oversampling ratio to $ OS = 1$ and vary the noise scale of the feature $c$
matrix in the set $\{10^{-4},10^{-3},10^{-2}\}$. For CGSI and FTCSI, their parameters $\alpha_i$ are all set
to $1$ and $\lambda$ is set to $0$.
The convergence behavior of the compared methods are reported in Fig.~\ref{fig:simulation}(d-f).
From these figures we can see that when converging, the RMSE
of CGSI and FTCSI are similar. This is because they solve the same problem. And even the feature matrices are
noisy, the RMSE of CGSI and FTCSI are much better than the other baselines. These figures also show that CGSI
is much faster than FTCSI, which is attributed to that CGSI is endowed with a better Riemannian metric.

\subsubsection{Case 3: influence of non-relevant features} We consider the performance of the proposed method, when the provided feature matrices $\FB_i$ have much more columns than the correct ones $\BB_i$. The matrices $\FB_i \in \RBB^{5000 \times 10(k+1)}$ is generated by
augmenting the correct feature matrices $\BB_i$ with $10k$ randomly generated columns. That is, we set $\FB_i = [\BB_i,\GB_i] + 10^{-5} \|\BB_i\|\EB_i$
where $\GB_i \in \RBB^{5000\times 10k}$ and $\EB_i\in\RBB^{5000 \times 10(k+1)}$ are random matrices with entries drew from i.i.d standard Gaussian distribution.
We fix the oversampling ratio to $OS = 1$, and vary the parameter $k\in\{10,30,50\}$. For
CGSI and FTCSI, $\alpha_i, 1\leq i \leq 3$ are set to $0.5$ and $\lambda$ is set to $0$. The parameters
of other baselines are set to the default. We
report the convergence behavior of the compared solvers in Fig.~\ref{fig:simulation} (g-i). From these figures we can see that both CGSI and FTCSI successfully bring the RMSE
down around $10^{-5}$ even when the columns of $\FB_i$ are 50 times larger than $\BB_i$.
And These figures also shows that the proposed solver CGSI converges much faster than FTCSI, which is attributed to CGSI being endowed with a better Riemannian metric.

\subsubsection{Case 4: influence of noisy samples} We consider the case where the observed entries are
noisy by adding a scaled Gaussian noise $\epsilon\PO(\ET)$ to  $\PO(\RT)$ where $\ET$ is a noise
tensor with i.i.d standard Gaussian entries. We fix the oversampling ratio $OS$ to $1$,  the noise scale $c$ of feature matrices to $10^{-4}$ and
vary the noise scale of samples such that $\epsilon \in \{10^{-4},10^{-3},10^{-2}\}$. For CGSI and FTCSI,
their parameters are set as follow.
When  $\alpha_i = 5, 1\leq i \leq 3$ and $\lambda = 0$. The parameters of other baselines are set to defaults.
We report the performance of the compared solvers in Fig.~\ref{fig:simulation} (j-l). From these figures we can see that only the solvers for the proposed model, that is CGSI and FTCSI,
bring the RMSE down to the level of noise  $\epsilon$ when converging. This shows that when the observed
entries are few, exploiting the side information can significantly improves the RMSE. Also we can see that CGSI converges much faster than
FTCSI, this exhibit that the proposed metric~(\ref{eqn:rMetric}) is able to accelerate the convergence of
Riemannian conjugate gradient descent method.
\begin{figure}
	\includegraphics[width = \columnwidth]{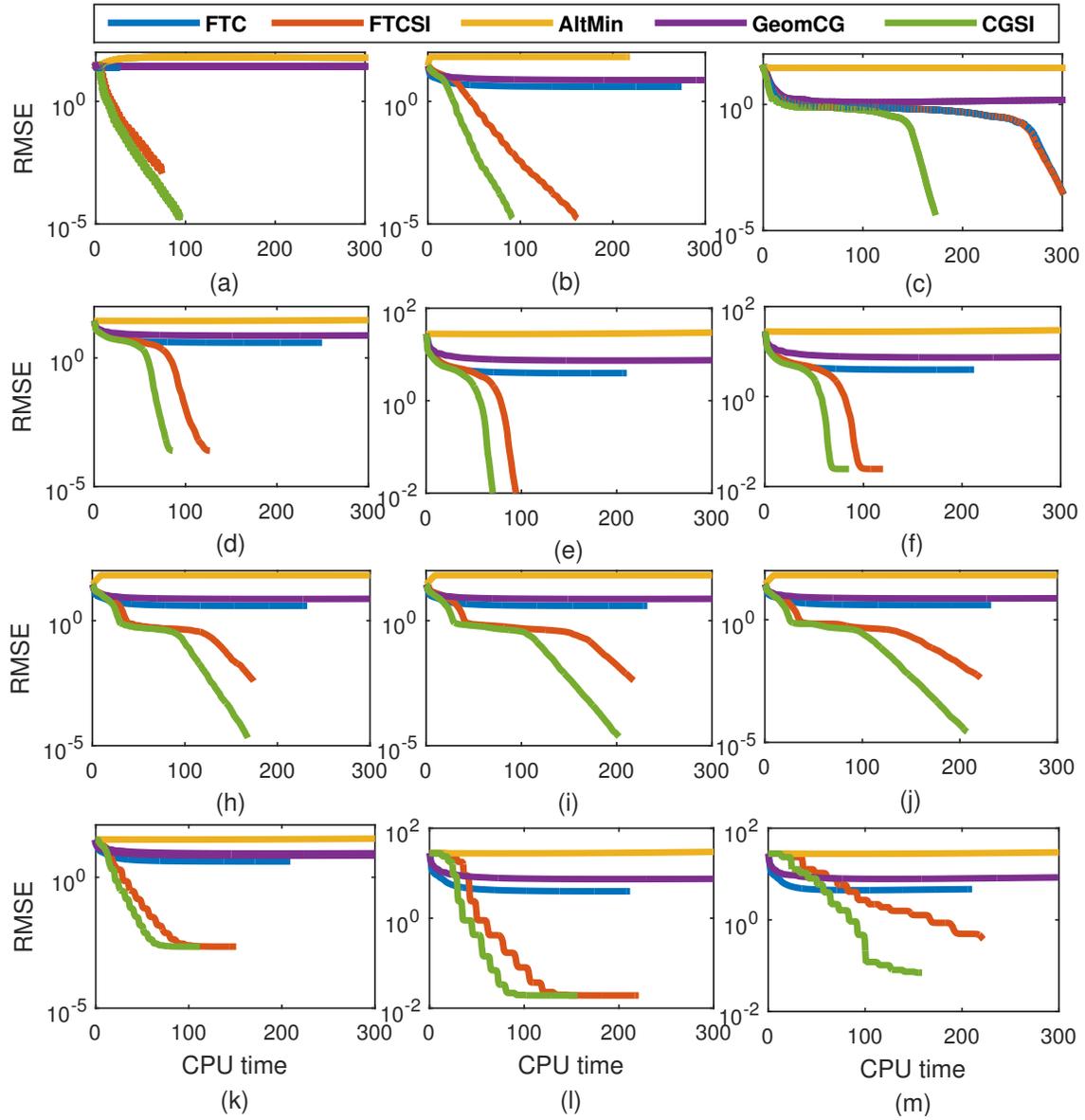}
	\caption{Simulation results of different solvers on the task of tensor completion.}
	\label{fig:simulation}
\end{figure}

\bibliographystyle{plainnat}
\bibliography{TenComp}

\end{document}